%% file: main.tex
\documentclass{article}

\usepackage[preprint]{neurips_2025}

\usepackage[utf8]{inputenc} %
\usepackage[T1]{fontenc}    %
\usepackage{hyperref}       %
\usepackage{url}            %
\usepackage{booktabs}       %
\usepackage{amsfonts}       %
\usepackage{nicefrac}       %
\usepackage{microtype}      %
\usepackage{xcolor}         %

\usepackage{algorithm}
\usepackage{algorithmic}
\newcommand{\RightComment}[1]{\hfill\algorithmiccomment{#1}}

\usepackage{amsmath}
\usepackage{amsfonts}
\usepackage{amssymb}
\usepackage{amsthm}
\usepackage{mathtools}
\usepackage{bm}
\usepackage{graphicx}
\usepackage{multirow}
\usepackage{wrapfig}
\usepackage{subfig}
\usepackage[export]{adjustbox}
\usepackage{tablefootnote}

\theoremstyle{plain}
\newtheorem{theorem}{Theorem}[section]
\newtheorem{proposition}[theorem]{Proposition}
\newtheorem{lemma}[theorem]{Lemma}
\newtheorem{coro}[theorem]{Corollary}
\theoremstyle{definition}

\newtheorem{assumption}[theorem]{Assumption}
\theoremstyle{remark}

\makeatletter
\renewenvironment{proof}[1][\proofname]{%
  \par\pushQED{\qed}%
  \noindent\textbf{#1.}\quad
}{%
  \popQED\par
}
\makeatother

\DeclareMathOperator{\argmin}{\arg\min}

\title{Autocorrelated Optimize-via-Estimate: \\ Predict-then-Optimize versus Finite-sample Optimal}

\date{}

\author{
  \textbf{Zichun Wang$^{1}$\thanks{Co-first authors}}, \,
  \textbf{Gar Goei Loke$^{2}$\footnotemark[1]},\,
  \textbf{Ruiting Zuo$^{1}$\thanks{Corresponding author}}\\
  $^1$The Hong Kong University of Science and Technology (Guangzhou)\\
  $^2$Durham University Business School
}

\begin{document}

\maketitle

\begin{abstract}
Models that directly optimize for out-of-sample performance in the finite-sample regime have emerged as a promising alternative to traditional estimate-then-optimize approaches in data-driven optimization. 
In this work, we compare their performance in the context of autocorrelated uncertainties, specifically, under a Vector Autoregressive Moving Average $\mathrm{VARMA}(p,q)$ process. 
We propose an autocorrelated Optimize-via-Estimate (A-OVE) model that obtains an out-of-sample optimal solution as a function of sufficient statistics, and propose a recursive form for computing its sufficient statistics. 
We evaluate these models on a portfolio optimization problem with trading costs. 
A-OVE achieves low regret relative to a perfect information oracle, outperforming predict-then-optimize machine learning benchmarks. 
Notably, machine learning models with higher accuracy can have poorer decision quality, echoing the growing literature in data-driven optimization. Performance is retained under small mis-specification.
\end{abstract}

\vspace{-7pt}
\input{chapters/1-intro}
\input{chapters/2-preliminaries}
\input{chapters/3-theory}

\input{chapters/4-experiments}

\input{chapters/5-conclusion}
{\small
\bibliographystyle{unsrtnat}
\bibliography{ref}
}
\appendix
\small
\newpage
\input{chapters/appendix}

\end{document}

%% file: chapters/1-intro.tex
\section{Introduction}
\label{sec:intro}
\subsection{Data-Driven Optimization}
Data-driven optimization is a specific but important instance of decision-making under uncertainty~\citep{Hertog_Postek_2016_bridging,mivsic2020data}, often formulated as the following stochastic program:
\begin{equation}\label{eq.nominal}
    \min_{\bm x\in \mathcal{B}} \; \rho(\bm{x}; D)
    := \min_{\bm x\in \mathcal{B}} \; \mathbb{E}_{\bm{y} \sim D} [\pi(\bm{x},\bm{y})],
\end{equation}
where decisions $\bm x$ in a feasible set $\mathcal{B}\subseteq \mathbb{R}^d$ optimize the expected value of a cost function $\pi(\bm x, \bm y)$, influenced by an uncertain variable $\bm y \in \mathcal{Y}\subseteq \mathbb{R}^n$ following true distribution $D$. In practice, true $D$ is unknown, instead a data set of $K$ historical observations sampled from $D$, denoted $\bm Y:= \{\bm y_i\}_{i=1}^K$, 
is available. A variant of this is the \emph{contextual} setting, where features (also called side information or covariates) related to unobserved uncertainty $\bm y$, denoted $\bm z$, are observed in a historic data set $\{(\bm z_i,\bm y_i)\}_{i=1}^K$ and also ahead of decision-making. This allows both (i) greater precision in estimating $\bm y$ via $\bm y | \bm z$, and (ii) decisions $\bm x := \bm x(\bm z)$ to be tailored to context $\bm z$.

A common approach is sequential learning and optimization: data set $\{(\bm z_i,\bm y_i)\}_{i=1}^K$ is used to learn relationship $\bm y|\bm z$, \emph{e.g.}, via machine learning, before an optimization model is solved using this relationship. There are two variants, namely predict-then-optimize (PTO) and estimate-then-optimize (ETO) \citep[e.g.,][]{ferreira2016analytics, liu2021time, Qi_Grigas_Shen_2021integrated}. PTO uses the learned relationship $\bm y | \bm z$ to form point estimate $\hat{\bm y}(\bm z)$ for the uncertainty, and optimizes \emph{deterministic} cost $\min_{\bm x} \pi(\bm{x},\hat{\bm y}(\bm z))$. Undoubtedly, this is biased vis-\`{a}-vis \eqref{eq.nominal} whenever $\pi$ is non-linear in $\bm y$ \citep{gupta2021small}. ETO avoids this by estimating the conditional distribution $\bm y|\bm z$ itself to optimize expected cost, for each context $\bm z$:
\begin{equation}\label{eq.contextual}
    \bm{x}(\bm z) \in \argmin\limits_{\bm x} \rho^C(\bm{x}) := \mathbb{E}_{\bm{y}|\bm z} [\pi(\bm{x},\bm{y})].
\end{equation}

Many learning schemes return only predictions, not distributions, as output and are inconsistent with \eqref{eq.contextual}. To solve this, \citet{bertsimas2020predictive} devised a method to generate a distribution from machine learning models for \eqref{eq.contextual}.

This two-stage paradigm is potentially suboptimal \citep{Liyanage_Shanthikumar_2005_practical, Mundru_2019predictive,hu2022fast}. Bafflingly, high prediction accuracy in the learning stage need not translate into good decision quality \citep{loke2022decision}; conversely, strong decisions can be attained using estimates with poor accuracy \citep{Elmachtoub2020smart}. This has motivated a growing body of literature on how to conduct learning and optimization \emph{in tandem}, termed integrated learning and optimization (ILO) \citep[see][for a comprehensive review]{Sadana_et_al_2024survey}.

The literature on ILO presents a wide range of approaches, from task-based learning \citep{Donti_Amos_Kolter_2017task} to robust optimization \citep{Kannan_Bayraksan_Luedtke_2024residuals, Sim_Tang_Zhou_Zhu_2024analytics}, from approaches that explicitly model the learning phase \citep{Elmachtoub2020smart} to those that do not \citep{Esteban_Morales_2022distributionally}. Most critically, while there have been some attempts to compare the performance of different models under a standardized setting \citep{tang2024pyepo}, there has yet to be conclusive comparisons across \emph{paradigms}, leading to an absence of theoretical explanations to what each paradigm is addressing and why they work.

At the same time, a recent stream of work on finite-sample optimality~\citep{feng2023framework,besbes2023big,LokeZhuZuo2023} has raised new questions in data-driven optimization. These works argue that decisions in data-driven optimization are functions of the data set, \emph{i.e.}, $\bm x(\bm Y)$, and therefore their performance should be evaluated in terms of \emph{out-of-sample} expectations taken over the sampling distribution, defined as follows: 
\begin{equation}\label{eq.oos}
    \sigma(\bm{x}(\cdot); D) = \mathbb{E}_{\bm Y} \left[\rho\left( \bm{x}(\bm Y); D \right) \right]. 
\end{equation}
Surprisingly, this out-of-sample objective can be directly and easily optimized! In numerical experiments, strong performance against state-of-the-art alternatives is observed. 

In this paper, we aim to study the potential of finite-sample optimal methods on a real-world problem, especially against machine learning models implemented in either PTO or ETO fashion, which may achieve lower prediction error. We focus on \citet{LokeZhuZuo2023}, as \citet{besbes2023big} is proposed specific to the newsvendor problem, whereas \citet{feng2023framework} restricts the allowed functional space for $\bm x(\cdot)$ to obtain solutions. Unfortunately, \citet{LokeZhuZuo2023} has yet been extended to the contextual setting. To mitigate this, we nominate the setting where uncertain $\bm y$ is a time-series. This allows the parametric approach of \citet{LokeZhuZuo2023} to be compared against machine learning methods, taking data of past times as covariates.

Specifically, we examine a portfolio optimization problem that accounts for trading costs. Trading costs are critical for practical portfolio implementation~\citep{kyle1985continuous, garleanu2016dynamic}, yet they are challenging to model directly as they depend on trade size and trader-specific characteristics. A common approach is to approximate trading costs using trading volume~\citep{frazzini2018trading, goyenko2024trading}, which exhibits strong temporal dependence to be modeled as a time series. Under this setting, \citet{goyenko2024trading} develop a portfolio optimization framework that balances tracking error against net-of-cost performance. They show that predicting trading volume via machine learning yields economic value comparable to stock return predictability, and their approach falls under the PTO paradigm. 

\subsection{Contributions and Findings}

Most primarily, we propose an autocorrelated model of \citet{LokeZhuZuo2023} (A-OVE) for the time-series setting, specifically, the vector autoregressive moving average ($\mathrm{VARMA}$) process. This extension is non-trivial due to the autocorrelated nature of the uncertainty. We then derive optimal solutions of A-OVE for a portfolio optimization problem with trading costs. In numerical experiments, we compare A-OVE against machine learning models under the PTO and ETO paradigms on both synthetic and real-world stock market datasets. Our main contributions and findings are:
\begin{enumerate}
    \item Theoretically, we derive the out-of-sample optimal solution for A-OVE as a function of sufficient statistics and propose a recursive form for computing the Fisher-Neyman decomposition of general $\mathrm{VARMA}(p,q)$ models, that has not been described in the literature to the best of our knowledge.
    \item A-OVE model achieves low regret relative to a perfect information oracle and consistently outperforms machine learning benchmarks---including neural networks, gradient boosting, and random forests---under both the PTO and ETO paradigms. These results are robust under a certain level of model mis-specification.
    \item Our numerical results further validate that machine learning models with higher predictive accuracy can yield poorer decision quality, reinforcing the importance of aligning learning objectives with downstream optimization tasks.
\end{enumerate}

%% file: chapters/2-preliminaries.tex
\section{Preliminaries and Notations}
\label{sec:preliminaries}
\subsection{Data-Driven Optimization Paradigms}

We present the data-driven optimization paradigms that we will examine in this paper.

\textbf{Predict-then-Optimize (PTO).} 
In the contextual setting, PTO learns a point predictor $\hat{\bm y}(\bm z)$ for the uncertain quantity $\bm y$ in each context $\bm z$ to optimizes for decisions $\bm x$ via
\begin{equation}\label{eq.pto}
    \bm x(\bm z) \in \argmin_{\bm x \in\mathcal{B}} \pi(\bm x, \hat{\bm y}(\bm z)). \tag{PTO}
\end{equation}
Even when predictor $\hat{\bm y}(\bm z) = \mathbb{E}[\bm y | \bm z]$ is unbiased, using $\pi(\bm x, \hat{\bm y}(\bm z))$ in place of $\mathbb{E}_{\bm{y}|\bm z} [\pi(\bm{x},\bm{y})]$ is biased in general if $\pi$ is non-linear in uncertainty $\bm y$. We denote PTO methods by their underlying predictive model. For example, PTO-RNN refers to a PTO model with predictors generated from a recurrent neural network.

\textbf{Estimate-then-Optimize (ETO).} 
ETO uses historic data to estimate the conditional distribution of the uncertainty given the context, denoted by $\hat D(\bm y | \bm z)$, and solves
\begin{equation}\label{eq.eto}
    \bm x(\bm z) \in \argmin_{\bm x \in\mathcal{B}} \mathbb{E}_{\bm{y}|\bm z \sim \hat{D}} [\pi(\bm{x}(\bm z),\bm{y})].  \tag{ETO}
\end{equation} 
We reserve the notation ``ETO'' to refer to the model that uses a parametric time-series learner for $\hat D(\bm y | \bm z)$. 

\textbf{From Predictive to Prescriptive (FPtP).}
Machine learning models cannot be directly applied within the ETO framework as they primarily output predictions, not conditional distributions. To bridge this, \citet{bertsimas2020predictive} propose the ``From Predictive to Prescriptive'' (FPtP) framework, which constructs a surrogate for the conditional expectation in \eqref{eq.eto}. This is done by generating $M$ predictive samples $\{\hat{\bm y}^{(m)}(\bm z)\}_{m=1}^M$ from the machine learning model and then assigning weights $\{w_m(\bm z)\}_{m=1}^M$, to solve
\begin{equation}\label{eq.fptp}
    \bm x(\bm z) \in \argmin_{\bm x \in\mathcal{B}} \sum\limits_{m=1}^M w_m(\bm z) \pi(\bm x, \hat{\bm y}_m(\bm z)). \tag{FPtP}
\end{equation}
How predictive samples and weights are generated is specific to the learner (implementation details for each learner are provided in Appendix~\ref{append.method_alg}. We denote FPtP models by their underlying learner. For instance, FPtP-RF denotes an FPtP approximation of ETO under a random forest learner. 

\textbf{Optimize-via-Estimate (OVE).} 
The OVE paradigm, proposed by \citet{LokeZhuZuo2023}, determines decision function $\bm x:= \bm x(\bm Y)$ of the data set $\bm Y$ by optimizing out-of-sample performance \eqref{eq.oos}. It adopts a parametric setting, assuming that $D$ belongs to a parametric family with density in $\mathcal{H}:=\{ f(\bm y; \bm \xi) : \bm \xi \in \Xi \subseteq \mathbb{R}^q \}$. As such, we write \eqref{eq.oos} as $\sigma(\bm{x}(\cdot); \bm \xi) = \mathbb{E}_{\bm Y} \left[\rho\left( \bm{x}(\bm Y); \bm \xi \right) \right]$, where the expectation is under the sampling distribution of $\bm Y$ with density denoted as $L\big(\bm Y; \bm \xi\big)$. At this point, $\sigma$ is ill-defined as $\bm \xi$ remains unknown. The OVE paradigm proposes a Bayesian-like approach to assume a prior $u(\bm \xi)$ for the unknown parameter $\bm \xi$, to optimize the posterior expected out-of-sample cost

\begin{equation}\label{eq.ove}
    \bm x \in \argmin\limits_{\bm x \in\mathcal{B}}\Psi(\bm{x}(\cdot); u):= \mathbb{E}_{\bm \xi \sim u} \left[\sigma\left( \bm{x}(\bm Y); \bm \xi \right) \right]. \tag{OVE}
\end{equation}

\begin{assumption}\label{asmp:convex_C}
    The feasible set $\mathcal B$ is convex. For all $\bm y\in\mathcal{Y}$, the cost function $\pi(\cdot,\bm y)$ is convex.
\end{assumption}

\begin{assumption}\label{asmp:strong_convexity}
    $\rho(\bm{x}; \bm \xi)$ is strictly convex and continuously differentiable for all $\bm \xi\in\Xi$, $\mathcal B$ is convex and closed, $\rho(\cdot; \bm \xi)$ is bounded below on $\mathcal B$ and $\lim\limits_{\|x\|\rightarrow\infty}\rho(\bm x; \bm \xi)=\infty$.
\end{assumption}

\begin{theorem}\label{thm.original_OVE}[\citet{LokeZhuZuo2023}]
    Let $\hat{\bm \xi} \in \hat{\Xi} \subseteq \mathbb{R}^r$ be any sufficient statistic for $\bm \xi$ with Fisher-Neyman factorization $L(\bm Y; \bm \xi) = g_0(\bm Y) g_1(\hat{\bm \xi}(\bm Y), \bm \xi)$ for non-negative $g_0$ and $g_1$. Let $\mathcal{X}(\hat{\bm \xi})$ be the set of feasible functions that are representable as a function of sufficient statistic $\hat{\bm \xi}$:
    \begin{equation}
        \mathcal{X}(\hat{\bm \xi}):= \left\{ \bm x \,\middle|\, \exists\, \tilde{\bm x},  \bm x(\bm Y) = \tilde{\bm x}(\hat{\bm \xi}(\bm Y)), \forall \bm Y \in \mathcal{Y}^K \right\}.
    \end{equation}
    Under Assumption \ref{asmp:convex_C}, if there is a re-parameterization $\mathcal{Y}^K \ni \bm Y \mapsto (\hat{\bm\xi},\bm Y|\hat{\bm\xi}) \in \hat\Xi\times\mathcal{Y}^K$ with non-negative Jacobian, then for fixed prior $u$,
    \begin{enumerate}\renewcommand{\labelenumi}{\roman{enumi}.}
        \item If an optimal solution for \eqref{eq.ove} exists, then one can be found in $\mathcal X(\hat{\bm\xi})$.
        \item The function $\bm x_{\scriptscriptstyle{\text{OVE}}}(\cdot)$, if it exists, defined pointwise on $\hat{\bm\xi}$, is optimal for \eqref{eq.ove}, i.e.,
        \begin{equation}\label{eq.ove_optimality}
            \bm x_{\scriptscriptstyle{\text{OVE}}}(\hat{\bm\xi}) \in \argmin_{\bm x\in\mathcal{B}} \beth(\bm x, \hat{\bm\xi}),
             \,\,\, \mbox{where} \,\,\,
             \beth(\bm x, \hat{\bm\xi}) = \mathbb{E}_{\bm \xi \sim u}\left[\rho(\bm x; \bm \xi) g_1(\hat{\bm\xi}, \bm \xi)\right].
        \end{equation}
        \item If furthermore Assumption \ref{asmp:strong_convexity} holds, then $\bm x_{\scriptscriptstyle{\text{OVE}}}(\cdot)$ exists and is uniquely optimal for \eqref{eq.ove}.
    \end{enumerate}
\end{theorem}

Theorem \ref{thm.original_OVE} allows optimal decisions to be solved as a function of the sufficient statistics of family $\mathcal{H}$, explicitly via \eqref{eq.ove_optimality}, which is a convex optimization problem in $\bm x$ that can be efficiently solved.

\subsection{Problem Formulation under Vector Autoregressive Moving Average Process}\label{subsec.time_series}

Let $\bm Y := (\bm Y_t)_{t=1}^T$ denote a length-$T$ realization of a zero-mean multivariate time series $\{\bm{Y}_t\}$, where $\bm{Y}_t = (Y_{1,t}, \ldots, Y_{n,t})^\top \in \mathcal{Y} \subseteq \mathbb{R}^n$. Each component $Y_{i,t}$ exhibits dependence across both time $t$ and component $i$.
We model such autocorrelated uncertainty using a Vector Autoregressive Moving Average ($\mathrm{VARMA}$) process. Specifically, $\bm{Y}$ follows a zero-mean $\mathrm{VARMA}(p,q)$ process, i.e., for every $t$,
\begin{equation}
    \bm{Y}_t - \sum_{i=1}^p \Phi_i \bm{Y}_{t-i} = \bm{\epsilon}_t + \sum_{i=1}^q \Theta_i \bm{\epsilon}_{t-i},
\label{eq.varmapq}
\end{equation}
where $\{\bm{\epsilon}_t\}$ is an i.i.d.\ white-noise innovation sequence with $\bm{\epsilon}_t \sim \mathcal{N}(\mathbf{0}, \Sigma_\epsilon)$, and $\Phi_i, \Theta_i \in \mathbb{R}^{n \times n}$ are coefficient matrices. In the $\mathrm{VARMA}$ context, it is natural to assume causality and invertibility, which ensure the process admits a well-defined infinite-order representation.
\begin{assumption}[Causality and Invertibility]\label{assump: causal_invert}
    For all $z \in \mathbb{C}$ with $|z| \leq 1$, $\Phi(z)$ and $\Theta(z)$ satisfy
    \begin{equation}\label{eq.causal_invert}
        \det\Phi(z) \neq 0 \quad \text{ and } \quad \det\Theta(z) \neq 0,
    \end{equation}
    where $\Phi(z) := I_n - \sum_{i=1}^p \Phi_i z^i$ and $\Theta(z) := I_n + \sum_{i=1}^q \Theta_i z^i$.
\end{assumption} 

As such, the autocorrelated uncertainty is characterized by the coefficient matrices $\bm \Phi = \{\Phi_i\}_{i=1}^p$, $\bm\Theta = \{\Theta_i\}_{i=1}^q $, and the covariance matrix $\Sigma_\epsilon$. We denote these parameters collectively as $\bm \xi := (\bm\Phi,\bm\Theta, \Sigma_{\epsilon})$, which are unknown. Then this autocorrelated setting can be adapted to the OVE framework. We term our proposed model \emph{Autocorrelated OVE} (A-OVE), where the optimal out-of-sample solution can be obtained via \eqref{eq.ove_optimality} as a function of sufficient statistics of $\bm \xi$, following Theorem~\ref{thm.original_OVE}. In Section~\ref{sec:theory}, we will further discuss the sufficient statistics for $\mathrm{VARMA}(p,q)$, which are nontrivial to derive due to the autocorrelation. Notably, rather than computing the sufficient statistics explicitly, we show that the A-OVE solution can be efficiently obtained via the likelihood function.
\subsection{Notation}
We adopt the notation $[T] = \{1,\dots,T\}$ and, for $i < T$, $[i,T] = {i,\dots,T}$. $\mathrm{Diag}(\cdot)$ denotes the operator that constructs a diagonal matrix from a vector, while $\mathrm{diag}(\cdot)$ extracts the diagonal entries of a matrix. 

%% file: chapters/3-theory.tex
\section{Model and Theory} \label{sec:theory}

In this Section, we first discuss the sufficient statistics of the $\mathrm{VARMA}(p,q)$ model. Next, we set up the portfolio optimization model with trading cost and propose an autocorrelated OVE model for determining optimal portfolio decisions.

\subsection{Sufficient Statistics for $\mathrm{VARMA}$}

Deriving sufficient statistics for the general $\mathrm{VARMA}(p,q)$ model is nontrivial due to autocorrelation. To the best of our knowledge, they have not been described in the literature. Sufficient statistics for univariate $\mathrm{ARMA}$ processes were derived in \citet{esa1998sufficient}, while \citet{kharrati2009sufficient} derived explicit forms for the multivariate $\mathrm{VARMA}(1,1)$ setting under fixed parameters and proposed approximate sufficient statistics for higher-order $\mathrm{VARMA}$ processes.

Sufficient statistics are central to OVE as Theorem \ref{thm.original_OVE} indicates. They also appear in the definition of $g_1$ in $\beth$ in \eqref{eq.ove_optimality}. However, note that the sufficient statistics themselves do not need to be computed in order to solve \eqref{eq.ove_optimality}. First, the sufficient statistics $\hat{\bm \xi}:= \hat{\bm \xi}(\bm Y)$ are a function of data set $\bm Y$, thus the OVE solution $\bm x_{\scriptscriptstyle{\text{OVE}}}(\hat{\bm \xi}(\bm Y)) = \bm x_{\scriptscriptstyle{\text{OVE}}}(\bm Y)$ can be represented in terms of data set $\bm Y$ itself. Second and more critically, $g_1$ is implicitly $g_1(\bm Y, \bm \xi)$. Thus, as long as one can compute $g_1$ for a given data set $\bm Y$ and parameter $\bm \xi$, \eqref{eq.ove_optimality} can still be optimized. 

To derive the likelihood function of a $\mathrm{VARMA}$ process, the following structural assumption is often assumed.

\begin{assumption}[Symmetric Covariance]\label{assump:sym_cov}
    The process $\bm{Y}$ has symmetric covariance, namely, for all $z \in \mathbb{C}$ with $|z| = 1$,  $\Phi(\cdot)$ and $\Theta(\cdot)$ satisfy
    \begin{equation}
        \Phi^{-1}(z)\Theta(z)\Sigma_\epsilon\Theta^{\top}(z^{-1})\Phi^{\top-1}(z^{-1})
        = \Phi^{-1}(z^{-1})\Theta(z^{-1})\Sigma_\epsilon\Theta^{\top}(z)\Phi^{\top-1}(z).
    \end{equation}
\end{assumption}

Under this assumption, the process admits a time-reversed representation, which enables the construction of a forward-transformed process $\{\bm W_t\}$ and, in turn, an exact likelihood formulation. The following theorem summarizes the resulting Fisher-Neyman decomposition.

\begin{theorem}[Fisher-Neyman Decomposition for general $\mathrm{VARMA}(p,q)$]\label{thm.varma_pq}
    Consider a $\mathrm{VARMA}(p,q)$ process $\bm Y$ with parameters $\bm \xi = ( \bm \Phi, \bm \Theta, \Sigma_{\epsilon})$ satisfying Assumptions~\ref{assump: causal_invert} and~\ref{assump:sym_cov}. Then there exists a forward transformation $\{\bm W_t\}$ and an associated one-step predictor $\{\hat{\bm{W}}_t\}$ such that the exact likelihood of $\bm Y$ admits the Fisher–Neyman decomposition
    \begin{equation}\label{eq.lkh_pq}
        L(\bm{Y};\bm{\xi})=g_0(\bm{Y}) g_1(\bm{Y},\bm{\xi}),
    \end{equation}
    where $g_0(\bm{Y}) = (2 \pi) ^ {- \frac {n T}{2}}$ and
    \begin{equation}\label{eq.g1_pq}
            g_1(\bm{Y},\bm{\xi}) = \left(\prod_ {t = 1} ^ {T} \det  (\Sigma_t) \right) ^ {- \frac {1}{2}} \exp \left\{- \frac {h(\bm Y,\bm{\xi})}{2} \right\}.
    \end{equation}
    Here, $h(\bm Y,\bm{\xi}) = \sum_ {t = 1} ^ {T} (\bm{W}_{t}-\hat{\bm{W}}_t)^\top  \Sigma_t^{- 1} (\bm{W}_{t}-\hat{\bm{W}}_t)$,
    and the covariance matrices $\Sigma_t$ of prediction error $\bm{W}_{t}-\hat{\bm{W}}_t$ are available in Appendix~\ref{append.proofs}.
\end{theorem}

\subsection{Portfolio Optimization with Trading Cost}
As motivated in the Introduction, we aim to compare different models on a portfolio optimization problem that balances risk-adjusted returns against trading costs. 

\textbf{Problem definition.} Consider a universe of $n$ assets and a fund manager with assets, valuing at $A$, at the start. At time $T+1$, the manager decides on a portfolio position $\bm x = (x_1, \ldots, x_n)^\top \in \mathbb{R}^n$ after observing historical information $\bm Y := (\bm Y_t)_{t=1}^T$ up to time $T$. The objective is to maximize
\begin{equation}\label{eq.utility}
     A(1+r^f) + \bm{e}^\top\bm{x} - \frac{\gamma}{2A}\bm{x}^\top\Sigma\bm{x}  - c(\bm{x},\bm{Y}_{T+1}),
\end{equation}
where the four terms respectively correspond to:  (i) the risk-free return at rate $r^f$, (ii) the expected excess returns at excess return rate $\bm{e}$, (iii) a mean–variance risk penalty with risk-aversion coefficient $\gamma$ and return covariance matrix $\Sigma$, and (iv) a trading cost term $c(\bm x, \bm Y_{T+1})$, which depends on both the portfolio decision and next-period log dollar trading volume $\bm Y_{T+1}$. Here, short-selling is permitted so portfolio decisions $\bm x$ are unconstrained ($\bm x < 0$ is permitted).

We specifically follow the setting of \citet{goyenko2024trading}, where uncertainty is assumed in the trading volume $\bm Y$ rather than the excess returns $\bm e$. Here, trading costs are modeled in a quadratic form: 
\begin{equation}
    c(\bm{x},\bm{Y}_{T+1}):= \frac{1}{2}(\bm{x}-\bm{x}^0)^\top\Lambda(\bm{Y}_{T+1})(\bm{x}-\bm{x}^0),
\end{equation}
where $\bm x^0$ denotes the initial portfolio position. The cost matrix $\Lambda(\bm{Y}_{T+1}):=\bar{\lambda} \mathrm{Diag}(\exp\{-\bm{Y}_{T+1}\})$ captures the inverse relationship between trading volume and transaction costs, with a scaling parameter $\bar{\lambda} > 0$. 

\textbf{Solving the models.}
Before detailing the solutions for each model, for analytical simplicity, we further assume $\Sigma = \delta^2 I_n$. Under this assumption, maximizing \eqref{eq.utility} is equivalent to solving the minimization problem
\begin{equation}\label{eq.pi}
    \min_{\bm{x}} \pi(\bm{x}; \bm{Y}_{T+1}) 
    := ( \bm{x}-\mu_0\bm{e} )^\top D_1 ( \bm{x}-\mu_0\bm{e} ) 
    + ( \bm{x}-\bm{x}^0 )^\top D_2(\bm{Y}_{T+1}) ( \bm{x}-\bm{x}^0 ),
\end{equation}
where $D_1 = \mu_1 \bm I_n$ and 
\begin{equation}\label{eq.D2}
    D_2(\bm{Y}_{T+1}) = \mu_2 \mathrm{Diag}(\exp\{-\bm{Y}_{T+1}\})
\end{equation}
with $\mu_0 = \frac{A}{\gamma\delta^2}$, $\mu_1=\frac{\gamma \delta^2}{2A}$ and $\mu_2 = \frac{\bar{\lambda} }{2}$.

\textbf{PTO.}
Given point prediction $\hat{\bm{Y}}_{T+1}:= \hat{\bm{Y}}_{T+1}(\bm Y)$, the optimal decision is
\begin{equation}\label{eq.x_pto}
    \bm x_{\scriptscriptstyle{\text{PTO}}} = \left( D_1 + D_2^{\scriptscriptstyle{\text{PTO}}} \right)^{-1} \left( \mu_0 D_1 \bm e + D_2^{\scriptscriptstyle{\text{PTO}}} \bm x^0 \right),
\end{equation}
where $D_2^{\scriptscriptstyle{\text{PTO}}} = D_2(\hat{\bm{Y}}_{T+1})$.

\textbf{ETO.} From \eqref{eq.pi}, the expected cost simplifies to
\begin{equation}\label{eq.rho}
        \rho(\bm x; \bm{\xi}) 
        = \left( \bm{x}-\mu_0\bm{e} \right)^\top D_1 \left( \bm{x}-\mu_0\bm{e} \right)
        + \left( \bm{x}-\bm{x}^0 \right)^\top \mathbb{E}_{\bm{Y_{T+1}}|\bm{\xi}} [D_2(\bm{Y}_{T+1})] \left( \bm{x}-\bm{x}^0 \right).
\end{equation}
Hence for a given set of estimated parameters $\hat{\bm{\xi}}_{\scriptscriptstyle{\text{ETO}}}:= (\hat{\Phi}, \hat{\Theta}, \hat{\Sigma}_{\epsilon})$, the optimal decision is
\begin{equation}\label{eq.x_eto}
    x_{\scriptscriptstyle{\text{ETO}}} 
    = \left( D_1 + D_2^{\scriptscriptstyle{\text{ETO}}} \right)^{-1} \left( \mu_0 D_1 \bm e + D_2^{\scriptscriptstyle{\text{ETO}}} \bm x^0 \right),
\end{equation}
where $D_2^{\scriptscriptstyle{\text{ETO}}} = \mathbb{E}_{\bm{Y}_{T+1}|\hat{\bm{\xi}}_{\scriptscriptstyle{\text{ETO}}}} [D_2(\bm{Y}_{T+1})]$. Moreover, note that $\bm{Y}_t \sim N(\bm{0}, \Gamma_{\bm{Y}_t|\bm{\xi}}(0))$ for all times $t\in[T+1]$. Denote $\bm{\gamma}^2(\bm{\xi}) = \mathrm{diag}\left(\Gamma_{\bm{Y}_{T+1}|\bm{\xi}}(0)\right)$. It follows that
\begin{equation}\label{eq.expected_D2}
    \mathbb{E}_{\bm{Y}_{T+1}|\bm{\xi}} [D_2(\bm{Y}_{T+1})] = \mu_2 \mathrm{Diag}\left(\exp\left\{\frac{\bm{\gamma}^2(\bm{\xi})}{2}\right\}\right).
\end{equation}
Details on estimating $\Gamma_{\bm Y_{T+1}}(0)$ are provided in Appendix~\ref{append.gamma_y}.

\textbf{FPtP-based ETO.}
The FPtP approach generates $M$ predictive samples $\hat{\bm{Y}}_{T+1} := (\hat{\bm{Y}}_{T+1}^{(m)})_{m=1}^M$ and weights $\bm w := (w_m)_{m=1}^M$, leading to optimal decisions
\begin{equation}\label{eq.x_fptp}
    \bm x_{\scriptscriptstyle{\text{FPtP}}} 
    = \left( D_1 + D_2^{\scriptscriptstyle{\text{FPtP}}} \right)^{-1} \left( \mu_0 D_1 \bm e + D_2^{\scriptscriptstyle{\text{FPtP}}} \bm x^0 \right),
\end{equation}
where $D_2^{\scriptscriptstyle{\text{FPtP}}} = \sum_{m=1}^M w_m D_2(\hat{\bm{Y}}_{T+1}^{(m)})$.

\textbf{A-OVE.} We propose an OVE model for the portfolio optimization problem under a $\mathrm{VARMA}(p,q)$ model specified by parameter $\bm \xi = (\bm \Phi, \bm \Theta, \Sigma_\epsilon)$. Under Assumptions~\ref{assump: causal_invert} and~\ref{assump:sym_cov}, we can apply Theorem \ref{thm.original_OVE}. 

\begin{coro}\label{coro.ove}
For any prior distribution $u \sim \bm{\xi}$, the unique optimizer of \eqref{eq.ove} has closed form solution
\begin{equation}\label{eq.x_ove}
    \bm x_{\scriptscriptstyle{\text{OVE}}}(\hat{\bm \xi}) = \left(\bar{D}_1^{\scriptscriptstyle{\text{OVE}}} + \bar{D}_2^{\scriptscriptstyle{\text{OVE}}} \right)^{-1}\left(\mu_0 \bar{D}_1^{\scriptscriptstyle{\text{OVE}}}\bm{e} + \bar{D}_2^{\scriptscriptstyle{\text{OVE}}} \bm{x}^0 \right), \tag{A-OVE}
\end{equation}
where
\begin{equation}\label{eq.D_ove}
     \bar{D}_1^{\scriptscriptstyle{\text{OVE}}} = c_1(\hat{\bm \xi}; u) D_1
     \quad \text{and} \quad
     \bar{D}_2^{\scriptscriptstyle{\text{OVE}}} =  C_2(\hat{\bm \xi}; u),
\end{equation}
with 
\begin{equation}\label{eq.C_ove}
    \begin{aligned}
        & c_1(\hat{\bm \xi}; u) :=  \int u(\bm \xi) g_1 ( \hat{\bm \xi} , \bm \xi ) d\bm{\xi} \\
        & C_2(\hat{\bm \xi}; u) := \int u(\bm \xi)\mathbb{E}_{\bm{Y}_{T+1}|\bm{\xi}} [D_2(\bm{Y}_{T+1})]  g_1 ( \hat{\bm \xi} , \bm \xi) d\bm \xi.
    \end{aligned}
\end{equation}
\end{coro}

Note that $\mathbb{E}_{\bm{Y}_{T+1}|\bm{\xi}} [D_2(\bm{Y}_{T+1})]$ can be computed for every $\bm \xi$ via \eqref{eq.expected_D2}. 
The quantities $c_1$ and $C_2$ can be approximated via Monte Carlo sampling from prior $u$, while the weighting function $g_1$ is available from \eqref{eq.g1_pq}. Algorithm~\ref{alg:ove} presents a pseudocode for solving \eqref{eq.x_ove}; implementation details for PTO, ETO, and FPtP are provided in Appendix~\ref{append.method_alg}. We note that the proposed algorithm for solving A-OVE can be extended to other optimization settings under $\mathrm{VARMA}(p,q)$ processes.

\input{algorithm/ove}

%% file: algorithm/ove.tex
\begin{algorithm}
\small
\caption{\textsc{MethodSolve}: A-OVE}
\label{alg:ove}
\begin{algorithmic}[1]
\REQUIRE
prior $u$;
number of sampled parameters of $\bm \xi$, $N_{\text{ove}}$;
sample length $T$;
sample $\bm{Y}$;
constants $\mu_0, D_1,\bm{e},\bm{x}^0$

\vspace{2mm}
\STATE Generate parameters $\bm{\xi}_i \sim u$, for $i=1,...,N_{\text{ove}}$

\FOR{$i=1$ {\bfseries to} $N_{\text{ove}}$} 
    \STATE $D_2^i  \gets \mathbb{E}_{\bm{Y}_{T+1}|\bm{\xi}_i} [D_2(\bm{Y}_{T+1})]$ 
    \RightComment{via \eqref{eq.expected_D2}}
    \STATE $g_1^i \gets g_1(\bm{Y},\bm{\xi}_i)$
    \RightComment{via \eqref{eq.g1_pq}}
    \STATE $C_2^i \gets D_2^i\cdot g_1^i$
\ENDFOR

\STATE $c_1 \gets \frac{1}{N_{\text{ove}}}\sum_{i=1}^{N_{\text{ove}}} g_1^i$
\RightComment{approximate $c_1$ in \eqref{eq.C_ove}}
\STATE $C_2 \gets \frac{1}{N_{\text{ove}}}\sum_{i=1}^{N_{\text{ove}}} C_2^i$
\RightComment{approximate $C_2$ in \eqref{eq.C_ove}}
\STATE $\bar{D}_1^{\scriptscriptstyle{\text{OVE}}} \gets c_1 D_1$
\STATE $\bar{D}_2^{\scriptscriptstyle{\text{OVE}}} \gets  C_2$
\STATE $\bm{x}_{\scriptscriptstyle{\text{OVE}}} \gets \left(\bar{D}_1^{\scriptscriptstyle{\text{OVE}}} + \bar{D}_2^{\scriptscriptstyle{\text{OVE}}} \right)^{-1}\left(\mu_0 \bar{D}_1^{\scriptscriptstyle{\text{OVE}}}\bm{e} + \bar{D}_2^{\scriptscriptstyle{\text{OVE}}} \bm{x}^0 \right)$

\vspace{2mm}
\STATE {\bfseries Return } $\bm{x}_{\scriptscriptstyle{\text{OVE}}}$
\end{algorithmic}
\end{algorithm}

%% file: chapters/4-experiments.tex
\section{Numerical Experiments}
\label{sec:experiments}

We evaluate the proposed A-OVE model against several benchmark methods on the portfolio optimization problem~\eqref{eq.rho}. In the first instance, we test the models under synthetic data, where the data-generating process is controlled. In the next subsection, we use real-world financial data, where modeling assumptions are only approximately satisfied. For simplicity, we assume $\mathrm{VARMA}(1,1)$ in our experiments. 

\textbf{Benchmark methods.} 
We consider four PTO models: PTO-RNN (recurrent neural network), PTO-LSTM (long short-term memory network), PTO-RF (random forest), and PTO-XGB (XGBoost). We additionally implement ETO variants under the FPtP implementation for tree-based models, FPtP-RF and FPtP-XGB. We also include a parametric ETO baseline that estimates a $\mathrm{VARMA}$ model via maximum likelihood. Among all methods, only A-OVE and parametric ETO explicitly exploit the assumed $\mathrm{VARMA}$ structure. For all machine learning models, they are allowed to use data of the past time periods of up to $S = 10 > p$ periods as covariates. A perfect-information oracle, which knows true $\bm \xi$, serves as a benchmark for computing relative regret.

\subsection{Numerical Results on Synthetic Data}

\textbf{Experimental design.}
We consider two settings. In the \emph{well-specified} setting, data are generated from a $\mathrm{VARMA}(1,1)$ model consistent with the assumptions used by A-OVE and ETO. In the \emph{misspecified} setting, data are generated from a higher-order $\mathrm{VARMA}(p,q)$ process, while A-OVE and ETO continue to assume a $\mathrm{VARMA}(1,1)$ specification.

To assess performance across a range of parameters, we draw $N_o$ oracle parameter vectors $\bm \xi_o$ independently from a prior distribution over the parameter space. For each oracle parameter, we generate $N_s$ independent time-series samples $\bm Y^o_i$ of length $T$. Each model generates a portfolio decision $\bm x^o_i(\bm Y^o_i)$ per sample, which is evaluated against the oracle's decision $\bm x^{o}_*$. Specifically, we compute the relative regret, $r^o_i = (\rho(\bm x^o_i(\bm Y^o_i); \bm \xi_o)-\rho(\bm x^{o}_*; \bm \xi_o))/\rho(\bm x^{o}_*; \bm \xi_o)$. The average relative regret for a given true parameter $\bm \xi_o$ can then be averaged across data sets: $R_o = \frac{1}{N_s} \sum_{i=1}^{N_s} r^o_i$. For each oracle, we approximate out-of-sample performance by averaging realized costs across samples, yielding an empirical estimate $\hat{\sigma}^o_{\text{syn}}$. This is compared against the oracle benchmark cost $\rho_*^o$, defined as the optimal cost under perfect information. Overall performance is summarized by the average relative regret across oracle parameters. Algorithm~\ref{alg:synthetic} provides a detailed description of the evaluation procedure.

\input{algorithm/synthetic}

In addition to decision performance, we evaluate predictive accuracy using the mean squared error (MSE), computed via a rolling evaluation scheme.
For further details on the simulation setup and experimental results, please refer to Appendix~\ref{append.exp.syn}.

\input{tables/well}

\begin{figure}[ht]
	\centering
	\includegraphics[trim=2mm 2mm 2mm 2mm, clip, width=0.7\linewidth]{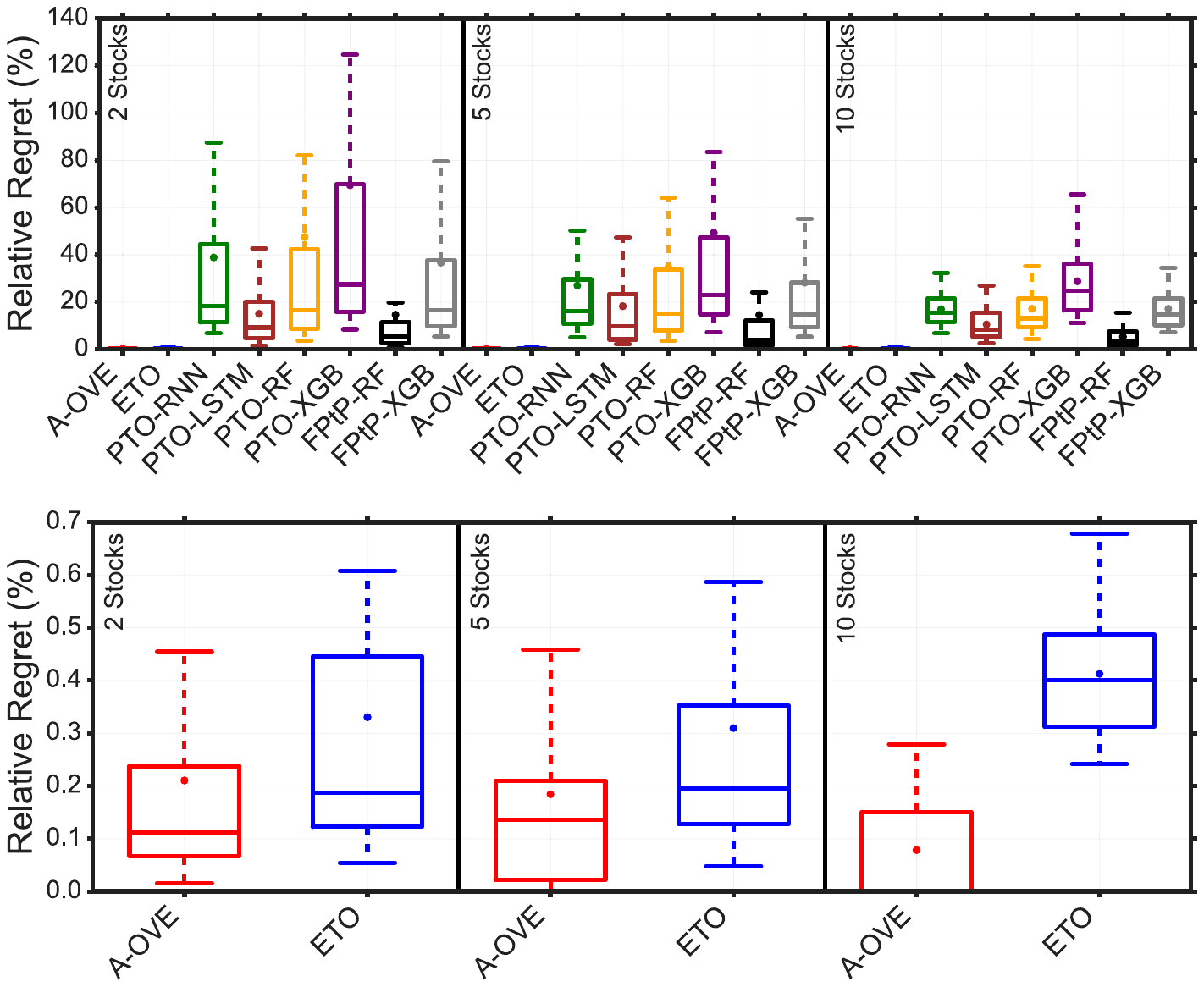}
    \vspace*{-2mm}
	\caption{Relative regret under well-specified setting.}
	\label{fig:11_relative_regret}
    \vspace{-.2cm}
\end{figure}

\textbf{Well-specified setting.} 
We first consider the portfolio problem~\eqref{eq.pi} with $n=2,5,$ and $10$ assets and oracle parameter $\bm{\xi} = (\Phi, \Theta, \Sigma_\epsilon)$ generated from a $\mathrm{VARMA}(1,1)$ process. Figure~\ref{fig:11_relative_regret} reports box plots of relative regret across oracle parameters, and  Table~\ref{tab:well} summarizes the results.  Across all dimensions, A-OVE consistently achieves the lowest average relative regret and exhibits the smallest variability, closely approaching oracle performance. Due to the well-specified setting, ETO shows good performance and has regret that is of the same order of magnitude as A-OVE. However, methods for estimating the parameters of $\mathrm{VARMA}$ models are notoriously unstable \citep{dufour2002linear}. The impact of such inaccuracies becomes apparent at $10$ stocks, where we can see significant degradation of performance of ETO. Model runtime also increases significantly.

Machine learning–based methods, both under PTO and FPtP implementations, incur substantially higher regret than A-OVE and ETO. This gap decreases as the number of assets increases, due to improvements in their predictive accuracy from higher information dimension. We also recover the observation in \citet{bertsimas2020predictive} that FPtP yields a substantial improvements over their PTO counterparts, confirming that optimizing with point estimates in $\pi$ introduces significant bias compared to optimizing the expected cost $\rho$. 

\begin{figure}[ht]
    \vspace{-.2cm}
    \centering
    \includegraphics[width=0.65\linewidth]{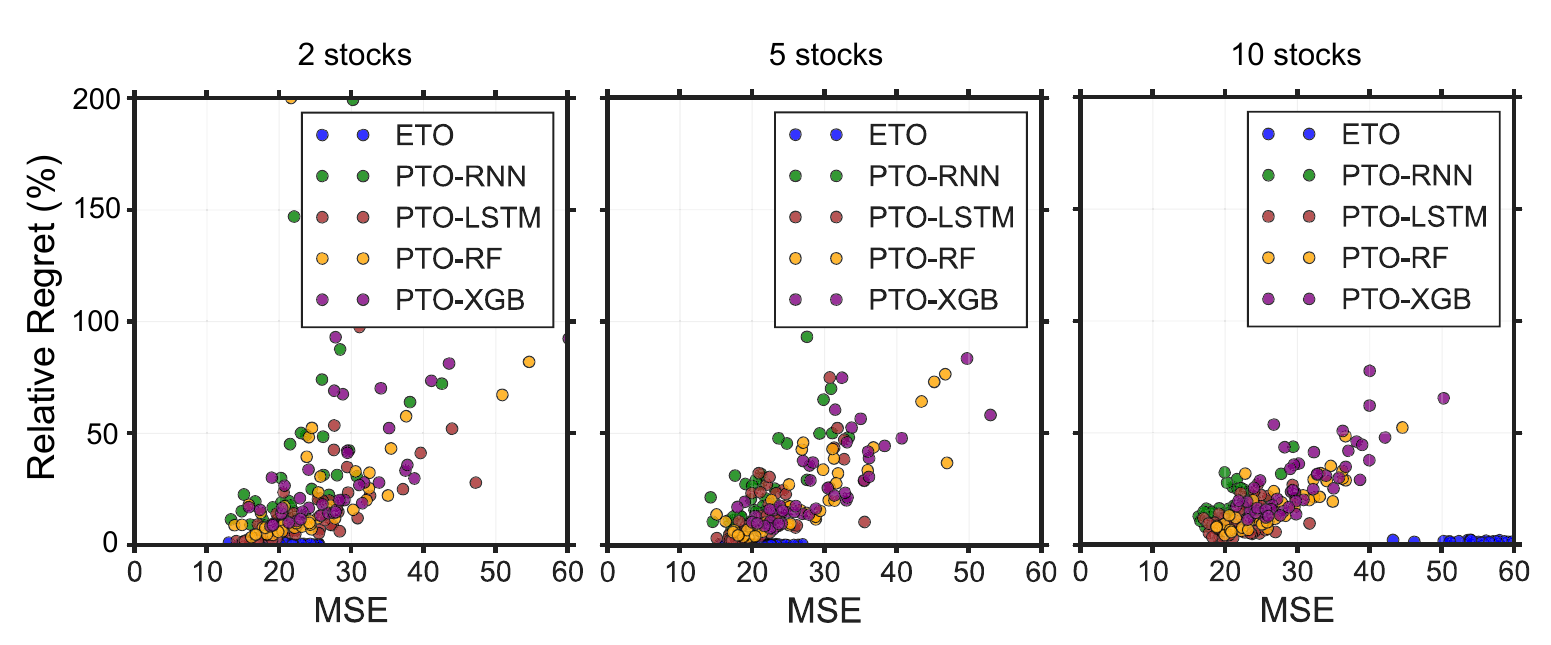}
    \vspace{-.2cm}
    \caption{Relative regret v.s. MSE under well-specified setting.}
    \label{fig:11_rr_mse}
    \vspace{-.2cm}
\end{figure}

Figure~\ref{fig:11_rr_mse} illustrates the relationship between relative regret and predictive accuracy across different problem dimensions. Within each machine learning method, the relative regret tends to decrease with predictive error and also as the dimension increases, implying that improving accuracy does help with improving regret. However, it is worthwhile to note that the method achieving lowest regret is not necessarily the method that has the best accuracy.  This observation highlights a fundamental misalignment between minimizing prediction error and optimizing downstream decision performance.

\begin{figure}[ht]
	\centering
	\includegraphics[trim=2mm 2mm 2mm 2mm, clip, width=0.65\linewidth]{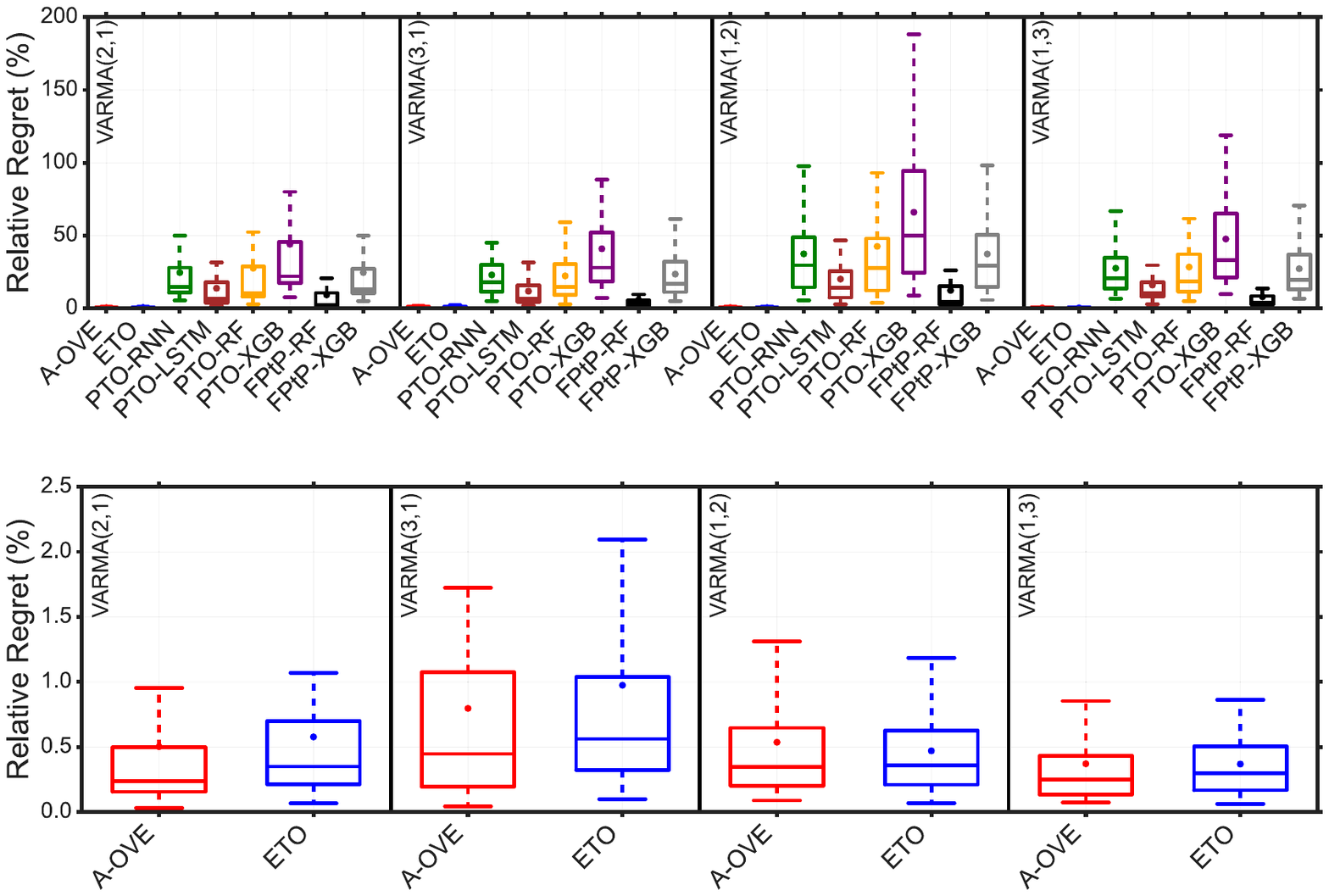}
    \vspace{-.2cm}
	\caption{Relative regret under mis-specified setting.}
	\label{fig:pq_relative_regret}
    \vspace{-.3cm}
\end{figure}

\textbf{Misspecified setting.}
We next evaluate robustness under small model mis-specification. We fix the number of stocks at $n=5$, because ETO becomes numerically unstable at $n=10$. We generate data from $\mathrm{VARMA}(p,q)$ processes outside the assumed $\mathrm{VARMA}(1,1)$ class, by varying one of $p$ or $q$ while holding the other fixed at one. Figure~\ref{fig:pq_relative_regret} reports relative regret under mis-specification. A-OVE continues to achieve the lowest regret in most settings, with ETO having regret of similar magnitude. Interestingly, machine learning methods which are not constrained by the $\mathrm{VARMA}$ setting and hence not vulnerable to mis-specification, do not manage to catch up to A-OVE and ETO. We further verify that their accuracies do indeed improve to the point of being significantly better than ETO in some cases, thus the stronger accuracy is insufficient in closing the gap in decision quality. 

ETO's unnatural stability even in the mis-specified case is mirrored in the literature \citep{Elmachtoub2020smart}. As A-OVE uses the information of the $\mathrm{VARMA}$ setting quite critically in its computation, via $g_1$ and $\rho$, it is expected that A-OVE will be more than proportionally affected by mis-specification. When mis-specification is in the parameter $q$, this effect is larger than with parameter $p$. We postulate that when $q$ is large, there is a larger noise reduction effect due to the moving average component, which allows ETO to perform better, as opposed to a degradation of A-OVE, which benefits from its accounting of noise through sampling effects. For more details about the results under the model mis-specification, please refer to Appendix~\ref{append.exp.syn}. 

\begin{figure}[htbp]
\vskip 0.1in
\begin{center}
\centerline{\includegraphics[width=0.95\textwidth]{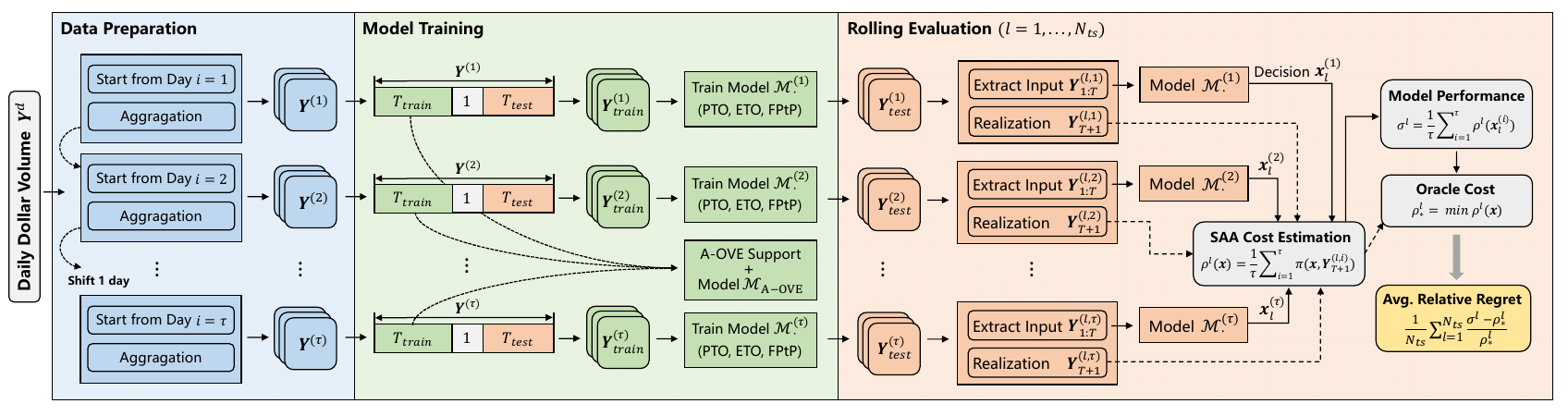}}
\caption{Evaluation on Real-World Data.}
\label{fig.real}
\end{center}
\vskip -0.3in
\end{figure}

\subsection{Numerical Results on Real-World Data}
We apply the proposed framework to real-world financial data using the World Stock Prices dataset. Our analysis focuses on the period from January~1,~2010 to January~1,~2019, excluding the COVID-19 period to avoid extreme volatility. In this experiment, we let each time period be a length of 2 weeks. We consider a portfolio of $n=4$ stocks—GOOGL, HMC, AXP, and BAMXF—whose log biweekly dollar trading volume is well described by a $\mathrm{VARMA}(1,1)$ process.

\textbf{Evaluation pipeline.}
As opposed to the case with synthetic data, we do not have the luxury of generating multiple $N_s$ data sets to control for sampling effects. Instead, we artificially construct different data sets for the same parameter by varying the start date when aggregating the total trading volume over the 2-week window, assuming that the underlying true parameter that generated the data remains the same across the 2-weeks. By doing so, we can supply different data sets that are assumed to be generated under the same true parameter to the models, as well as compute performance under the expected cost $\rho$, rather than depend on the deterministic cost $\pi$. 

Details of the implementation from data preparation, model training, and rolling out-of-sample evaluation is described in Figure~\ref{fig.real}. Specifically, daily dollar trading volume $\bm Y^d$ is aggregated into $\tau$ shifted log biweekly series using a rolling window. For each aggregation sequence $(\bm Y^{(i)})_{i=1}^{\tau}$, a collection of models $\mathcal{M}$ is trained using the training data, while A-OVE constructs its parameter support using all available training data. The out-of-sample performance is evaluated using a rolling-window scheme. At each evaluation step $l = 1,\dots,N_{ts}$, we extract $\tau$ input–output pairs ${(\bm Y^{(l,i)}_{1:T}, \bm Y^{(l,i)}_{T+1})}_{i=1}^{\tau}$ from the testing data. Given these realizations, we define the empirical cost function $\rho^{l}(\bm x) := \frac{1}{\tau}\sum_{i=1}^{\tau} \pi(\bm x, \bm Y^{(l,i)}_{T+1})$, which serves as a SAA of the true expected cost at evaluation step $l$. Each trained model produces a decision $\bm x_l^{(i)}$ based on the input $\bm Y^{(l,i)}_{1:T}$. The performance of a model is then measured by averaging its realized costs across all $\tau$ aggregations. In addition, an oracle decision $\bm x_*^l$ is computed by directly minimizing $\rho^l(\bm x)$. We report the average relative regret across all rolling evaluation windows $\frac{1}{N_{ts}}\sum_{l=1}^{N_{ts}} \frac{\sigma_\cdot^l - \rho_*^l}{\rho_*^l}$ as the out-of-sample performance of each model relative to the oracle.

\textbf{Results.}
Table~\ref{tab:real} reports predictive accuracy, computation time, and relative regret of different methods on real-world financial data. A-OVE achieves consistently low relative regret, comparable to or better than the strongest prediction-based methods. Interestingly, PTO methods now exhibit a large variation in performance to choice of predictive model. Neural network models now show strong performance against other machine learning models in the PTO regime, though FPtP-RF still performs best amongst machine learning model. Neural network models also have significantly lower MSE compared to the other learners, but do not necessarily have lower regret compared to some models that have twice the predictive error, once again affirming the understanding that predictive accuracy does not guarantee decision quality, even in the real-world data setting.

\input{tables/real}

%% file: algorithm/synthetic.tex
\begin{algorithm}
\caption{Evaluation on Synthetic Data}
\small
\label{alg:synthetic}
\begin{algorithmic}[1]
\REQUIRE
number of oracles $N_o$;
number of samples per oracle $N_s$;
oracle parameter set $\Xi^{\text{ora}}$;
sample length $T$;
sample set $\{\mathcal{Y}_o\}_{o=1}^{N_o}$;
constants $\mu_0, D_1,\bm{e},\bm{x}^0$

\vspace{2mm}
\FOR{$\bm{\xi}_o$ {\bfseries in} $\Xi^{\text{ora}}$} 
    \STATE $\bm x^o_* \gets \argmin_{\bm{x}} \rho(\bm{x}; \bm{\xi}_o)$
    \STATE $\rho_*^o \gets \rho(\bm x^o_*; \bm{\xi}_o)$
    \vspace{1mm}
    \FOR{$\bm Y^o_i$ {\bfseries in} $\mathcal{Y}_o$}
        \STATE $\bm{x}^o_i \gets $ \textsc{MethodSolve}$(\bm Y^o_i)$
        \STATE $\rho^o_{i} \gets \rho(\bm{x}^o_i; \bm{\xi}_o)$
        \RightComment{via \eqref{eq.rho}}
        \STATE $r^o_i \gets (\rho^o_{i}-\rho_*^o)/\rho_*^o$
    \ENDFOR
    \vspace{1mm}
    \STATE $\hat{\sigma}^o_{\text{syn}} \gets \frac{1}{N_s} \sum_{i=1}^{N_s}  \rho_{i}$
    \RightComment{approximation of \eqref{eq.oos}}
    \STATE $R_o \gets \frac{1}{N_s} \sum_{i=1}^{N_s} r^o_i$
\ENDFOR

\vspace{2mm}
\STATE {\bfseries Return } $\frac{1}{N_o}\sum_{o=1}^{N_o} R_o$
\end{algorithmic}
\end{algorithm}

%% file: tables/well.tex
\begin{table*}[htpb]
\caption{Summary results of well-specified setting.}
\label{tab:well}
\centering
\renewcommand{\arraystretch}{1.1}
\begin{small}
\begin{sc}
\begin{tabular}{lcccccccccc}
\toprule
\multicolumn{1}{c}{\multirow{2}{*}{Model}} & \multicolumn{3}{c}{Time (second)} & \multicolumn{3}{c}{MSE} & \multicolumn{3}{c}{Relative Regret (\%)} \\
\cmidrule(lr){2-4} \cmidrule(lr){5-7} \cmidrule(lr){8-10}
 & $n=2$ & $n=5$ & $n=10$ & $n=2$ & $n=5$ & $n=10$ & $n=2$ & $n=5$ & $n=10$ \\ \hline
A-OVE & 1.56 & 3.18 & 3.37 & - & - & - & \textbf{0.21} & \textbf{0.18} & \textbf{0.08} \\
ETO & 1.20 & 4.45 & 13.44 & \textbf{20.05} & \textbf{21.21} & 59.56 & 0.33 & 0.31 & 0.41 \\ \hline
PTO-RNN & \textbf{0.25} & \textbf{0.27} & 0.43 & 23.14 & 21.32 & \textbf{20.86} & 38.85 & 27.00 & 16.95 \\
PTO-LSTM & 0.31 & 0.31 & \textbf{0.33} & 23.80 & 22.36 & 21.99 & 15.02 & 18.23 & 10.55 \\
PTO-RF & 0.43 & 0.44 & 0.45 & 26.22 & 27.90 & 26.96 & 47.53 & 34.62 & 17.15 \\
PTO-XGB & 1.30 & 3.06 & 5.25 & 30.29 & 31.68 & 30.50 & 69.46 & 49.27 & 28.81 \\ \hline
FPtP-RF & 0.43 & 0.44 & 0.45 & 26.22 & 27.90 & 26.96 & 14.66 & 14.58 & 5.33 \\
FPtP-XGB & 1.30 & 3.06 & 5.25 & 30.29 & 31.68 & 30.50 & 36.69 & 28.14 & 17.15 \\ 
\bottomrule
\end{tabular}
\end{sc}
\end{small}
\end{table*}

%% file: tables/real.tex
\begin{table}[ht]
\vskip -0.1in
\caption{Summary results of real-world testing.}
\label{tab:real}
\vskip -0.4in
\begin{center}
\begin{small}
\begin{sc}
\begin{tabular}{lcccr}
\toprule
Method & Time(s) & MSE & Relative Regret (\%) \\
\midrule
A-OVE     & 0.19& -   & \textbf{0.02} \\
ETO       & \textbf{0.00}& 0.53& \textbf{0.02} \\
PTO-RNN   & \textbf{0.00}& 0.25& 0.04 \\
PTO-LSTM  & \textbf{0.00}& \textbf{0.23}& 0.03 \\
PTO-RF    & 0.18& 0.44& 1.89 \\
PTO-XGB   & 0.03& 1.36& 327.82 \\
FPtP-RF   & 0.18& 0.44& \textbf{0.02} \\
FPtP-XGB  & 0.03& 1.36& 1.94 \\
\bottomrule
\end{tabular}
\end{sc}
\end{small}
\end{center}
\vskip -0.2in
\end{table}

%% file: chapters/5-conclusion.tex
\section{Conclusion and Broader Impact}
\label{sec:conclusion}

Our results indicate a few important principles in contextual optimization. First, predictive accuracy does not guarantee good decisions. Often, the most accurate model does not give the lowest regret and sometimes the model with the best decisions can have double the predictive error as the most accurate model. This is verified not just in simulation, but even under experiments with real-world data. This does question the current directions of machine learning to put predictive accuracy at its core. Many decisions in the real world are not predictive tasks. While prediction contributes to decision-making, their relationship can be non-linear. At the same time, relative to exploration into learning models, the community focusing on contextual and end-to-end learning is relatively small. 

Second, we believe that our experiments have presented clearer evidence of the bias induced when solving PTO as opposed to ETO. In all our experiments, FPtP significantly reduces the regret of their PTO counterparts, without any significant penalty to computational times. Moving forward, we recommend utilizing FPtP in place of PTO by default when making decisions under uncertainty.

Third, our experiments illustrate the clear potential of out-of-sample methods, such as \citet{feng2023framework}, \citet{besbes2023big} and \citet{LokeZhuZuo2023}. Conceivably, OVE would perform even better when it is given the ability to react to the context. We hope that our work here will draw more attention to this family of methods, and their associated techniques.

Several directions for future research emerge from this work.
On the theoretical side, there is a lot of scope to extend the existing literature on out-of-sample methods. From our experiments, we believe there is significant value if a contextual version of OVE could be designed. Additionally, OVE did indicate a drop in performance as a result of mis-specification. Such drawbacks could be potentially mitigated if a non-parametric version of OVE was developed, thus liberating it from its reliance on its parametric assumptions. Notably, Theorem \ref{thm.varma_pq} already indicates that the dependence of the theory of \citet{LokeZhuZuo2023} on sufficient statistics can potentially be diminished to some extent. 

From a practical perspective, our work is only a singular test of out-of-sample optimal methods on the portfolio optimization problem. It would be useful to examine if our conclusions extend beyond the current setting.

%% file: chapters/appendix.tex
\section{Supplementary Material: Deferred Proofs}
\label{append.proofs}
This section presents proofs of the results stated in Sections~\ref{sec:preliminaries} and \ref{sec:theory}. For the reader’s convenience, each result is restated prior to its proof.

\textbf{Theorem~\ref{thm.original_OVE}.}
\begingroup\itshape
    Let $\hat{\bm \xi} \in \hat{\Xi} \subseteq \mathbb{R}^r$ be any sufficient statistic for $\bm \xi$ with Fisher-Neyman factorization $L(\bm Y; \bm \xi) = g_0(\bm Y) g_1(\hat{\bm \xi}(\bm Y), \bm \xi)$ for non-negative $g_0$ and $g_1$. Let $\mathcal{X}(\hat{\bm \xi})$ be the set of feasible functions that are representable as a function of sufficient statistic $\hat{\bm \xi}$:
    \[
    \mathcal{X}(\hat{\bm \xi}):= \left\{ \bm x : \mathcal{Y}^K \rightarrow \mathcal{B} \,\middle|\, \exists\, \tilde{\bm x} : \hat\Xi \rightarrow \mathcal{B},  \bm x(\bm Y) = \tilde{\bm x}(\hat{\bm \xi}(\bm Y)), \forall \bm Y \in \mathcal{Y}^K \right\}.
    \]
    
    If there is a re-parameterization from $\bm Y \in \mathcal{Y}^K$ to $(\hat{\bm\xi},\bm Y|\hat{\bm\xi}) \in \hat\Xi\times\mathcal{Y}^K$ with non-negative Jacobian, denoted $J(\hat{\bm\xi},\bm Y|\hat{\bm\xi})$, for all $\hat{\bm\xi}\in\hat\Xi$ and $\bm Y|\hat{\bm\xi}\in\mathcal{Y}^K$, then for fixed prior $u$, under Assumption \ref{asmp:convex_C},
    \begin{enumerate}
        \item If an optimal solution for \eqref{eq.ove} exists, then one can be found in $\mathcal X(\hat{\bm\xi})$.
        \item The function $\bm x_{\scriptscriptstyle{\text{OVE}}}$, if it exists, defined pointwise on $\hat{\bm\xi}$, is prior optimal, i.e.,
        \[
        \bm x_{\scriptscriptstyle{\text{OVE}}}(\hat{\bm\xi}) \in \argmin_{\bm x\in\mathcal{B}} \beth(\bm x, \hat{\bm\xi}) 
        ,\quad \mbox{where} \quad 
        \beth(\bm x, \hat{\bm\xi}) = \mathbb{E}_{\bm \xi \sim u}\left[\rho(\bm x; \bm \xi) g_1(\hat{\bm\xi}, \bm \xi)\right].
        \]
        \item If furthermore Assumption \ref{asmp:strong_convexity} holds, then $\bm x_{\scriptscriptstyle{\text{OVE}}}$ exists and is unique.
    \end{enumerate}
\endgroup
\begin{proof}[Proof of Theorem~\ref{thm.original_OVE}]
    The proof is available from \citet{LokeZhuZuo2023}. 
\end{proof}
\textbf{Theorem~\ref{thm.varma_pq}.} 
\begingroup\itshape
Consider a $\mathrm{VARMA}(p,q)$ process $\bm Y$ with parameters $\bm \xi = ( \bm \Phi, \bm \Theta, \Sigma_{\epsilon})$ satisfying Assumptions~\ref{assump: causal_invert} and~\ref{assump:sym_cov}. Then there exists a forward transformation $\{\bm W_t\}$ and an associated one-step predictor $\{\hat{\bm{W}}_t\}$ such that the exact likelihood of $\bm Y$ admits the Fisher–Neyman decomposition
    \[
        L(\bm{Y};\bm{\xi})=g_0(\bm{Y}) g_1(\bm{Y},\bm{\xi}),
    \]
    where
    \[
        g_0(\bm{Y}) = (2 \pi) ^ {- \frac {n T}{2}}, \quad
        g_1(\bm{Y},\bm{\xi}) = \left(\prod_ {t = 1} ^ {T} \det  (\Sigma_t) \right) ^ {- \frac {1}{2}} \exp \left\{- \frac {h(\bm Y,\bm{\xi})}{2} \right\}.
    \]
    Here, 
    \[
        h(\bm Y,\bm{\xi}) = \sum_ {t = 1} ^ {T} (\bm{W}_{t}-\hat{\bm{W}}_t)^\top  \Sigma_t^{- 1} (\bm{W}_{t}-\hat{\bm{W}}_t),
    \]
    and the covariance matrix $\Sigma_t$ of prediction error $\bm{W}_{t}-\hat{\bm{W}}_t$ are available in Appendix~\ref{append.proofs}.
\endgroup
\begin{proof}[Proof of Theorem~\ref{thm.varma_pq}]
Following the innovation algorithm for applying a $\mathrm{VARMA}(p,q)$ model to the backward‐transformed process proposed by \citet{brockwell1991time}, we adopt the approach of \citet{esa1998sufficient} and \citet{kharrati2009sufficient} to derive an innovation algorithm for the forward‐transformed process. This enables us to obtain the likelihood function of the $\mathrm{VARMA}(p,q)$ model.

We begin with a $\mathrm{VARMA}(p,q)$ process that satisfies Assumptions~\ref{assump: causal_invert} and~\ref{assump:sym_cov}. Under these conditions, the process admits an equivalent forward representation:
\[
\bm{Y}_t-\sum_{i=1}^p\Phi_i \bm{Y}_{t+i}=\bm{\tilde{\epsilon}}_t+\sum_{i=1}^q\Theta_i \bm{\tilde{\epsilon}}_{t+i},
\]
where $\{\bm{\tilde{\epsilon}}_t\}\sim \text{WN}(\mathbf{0},\,\Sigma_\epsilon)$, sharing the same distribution as $\{\bm{\epsilon}_t\}$. For notational simplicity, we continue to denote this white noise sequence by ${\bm{\epsilon}_t}$.

Let $l = \max\{p,q\}$ and define the transformation $\bm{W}_t$ as:
\[
    \bm{W}_t=
    \begin{cases}
        \bm{Y}_t-\sum_{i=1}^p \Phi_i \bm{Y}_{t+i} &\quad t\in[T-l],\\
        \bm{Y}_t &\quad t\in[T-l+1,T].
    \end{cases}
\]

Lemma~\ref{lemma:Gamma_W} characterize its covariance structure and then using this result, we can derive the innovation algorithm for the forward-transformed process, which is stated in Lemma~\ref{lemma:inno_alg_W}. It is worth noting that the Jacobian of the transformation from $\bm{Y}_t$ to $\bm{W}_t$ equals one. Consequently, the sequence $\{\bm{W}_t\}$ preserves the full probabilistic information contained in $\bm Y$. The exact likelihood function of the transformed process $\{\bm{W}_t\}$ is therefore given by
\[
    L (\bm{Y};\bm\xi) 
    = (2 \pi) ^ {- \frac {n T}{2}} \left(\prod_ {t = 1} ^ {T} \det  (\Sigma_t)\right) ^ {- \frac {1}{2}}
    \times \exp \left\{- \frac {1}{2} \sum_ {t = 1} ^ {T} \bm{R}_{t}^\top  \Sigma_t^{- 1} \bm{R}_{t}\right\}
\]
where $\bm{R}_{t} = \bm{W}_{t}-\hat{\bm{W}}_t$. Applying the Fisher-Neyman factorization theorem, the likelihood function can be expressed in the decomposed form
\[
L(\bm{Y};\bm{\xi})=g_0(\bm{Y}) g_1(\bm{Y},\bm{\xi}),
\]
where
\[
    g_0(\bm{Y}) = (2 \pi) ^ {- \frac {n T}{2}}, \quad
    g_1(\bm{Y},\bm{\xi}) = \left(\prod_ {t = 1} ^ {T} \det  (\Sigma_t) \right) ^ {- \frac {1}{2}} \times \exp \left\{- \frac {1}{2} \sum_ {t = 1} ^ {T} \bm{R}_{t}^\top  \Sigma_t^{- 1} \bm{R}_{t}\right\}.
\]

\end{proof}

\begin{lemma}\label{lemma:Gamma_W}
    Let $\Gamma_{\bm{W}}(t,t-h)$ denote the covariance between $\bm{W}_t$ and its lag-$h$ value for $h\geq 0$. Then
    \begin{equation}\label{eq.Gamma_W}
        \Gamma_{\bm{W}}(t,t-h)=
            \begin{cases}
                \Gamma_{\bm{Y}}(h), & 0\leq h<l,\, T-l<t\leq T,\\[12pt]
                \displaystyle
                \Gamma_{\bm{Y}}(h)-\sum_{i=1}^{p}\Phi_i\,\Gamma_{\bm{Y}}(h-i), & t+l-T\leq h< t+2l-T,\, T-l<t\leq T,\\[12pt]
                \displaystyle
                \sum_{i=0}^{q-h}\Theta_i\Sigma_\epsilon\Theta_{i+h}^{\top}, & 0\leq h\leq q,\, 1\leq t\leq T-l,\,\\[12pt]
                \bm{0}_{n\times n}, & \text{otherwise.}
            \end{cases}          
    \end{equation}
    where $\Gamma_{\bm{Y}}(h)$ is the covariance matrix of $\bm{Y}_t$ at lag $h$.
\end{lemma}

\begin{proof}[Proof of Lemma~\ref{lemma:Gamma_W}]
We evaluate $\Gamma_{\bm{W}}(t,t-h)$ by considering several cases:
\begin{itemize}
    \item If $T-l < t-h \leq t \leq T$, then $\Gamma_{\bm{W}}(t,t-h) = Cov(\bm{Y}_t, \bm{Y}_{t-h})=\Gamma_{\bm{Y}}(h)$.
    
    \item If $T-2l < t-h \leq T-l < t \leq T$, then
    \[
    \Gamma_{\bm{W}}(t,t-h) = Cov(\bm{Y}_t, \bm{Y}_{t-h}-\sum_{i=1}^p \Phi_i \bm{Y}_{t-h+i} )=\Gamma_{\bm{Y}}(h) - \sum_{i=1}^p \Phi_i \Gamma_{\bm{Y}}(h-i)
    \]
    
    \item If $t-q \leq t-h \leq t \leq T-l$, then 
    \[
    \begin{aligned}
        \Gamma_{\bm{W}}(t,t-h) 
        &= Cov(\bm{Y}_t-\sum_{i=1}^p \Phi_i \bm{Y}_{t+i}, \bm{Y}_{t-h}-\sum_{i=1}^p \Phi_i \bm{Y}_{t-h+i} ) \\
        &= Cov(\bm{\epsilon}_t + \sum_{i=1}^q \Theta_i \bm{\epsilon}_{t+i}, \bm{\epsilon}_{t-h} + \sum_{i=1}^q \Theta_i \bm{\epsilon}_{t-h+i}) \\
        &= \sum_{i=0}^{q-h} \Theta_i \Sigma_\epsilon \Theta_{i+h}^\top
    \end{aligned}
    \]
    where $\Theta_0=\mathbf I_{n\times n}$
    \item Otherwise, $\Gamma_{\bm{W}}(t,t-h)=\bm{0}_{n \times n}$
\end{itemize}

By reorganizing the ranges for the cases above, we obtain the expression stated in Lemma~\ref{lemma:Gamma_W}.
\end{proof}

\begin{lemma}[The Innovation Algorithm of $\{\bm{W}_t\}$]
\label{lemma:inno_alg_W}
The one-step predictor of the transformed process $\{\bm{W}_t\}$ are given by 
\begin{equation}
    \hat{\bm{W}}_t=
    \begin{cases}
        0 &\quad t=1,\\
        \sum_{j=1}^{t-1} \Theta_{t-1,j}(\bm{W}_{t-j}-\hat{\bm{W}}_{t-j})&\quad t\in[2,T],
    \end{cases}
\end{equation}
and the covariance matrix $\Sigma_t$ of the prediction error $\bm{R}_{t}=\bm{W}_{t}-\hat{\bm{W}}_{t}$ are given by
\begin{equation}
    \Sigma_t=
    \begin{cases}
        \Gamma_{\bm{W}}(1,1) &\quad t=1,\\
        \Gamma_{\bm{W}}(t,t) - \sum_{j=1}^{t-1}\Theta_{t-1,j} \Sigma_{t-j} \Theta_{t-1,j}^\top &\quad t\in[2,T],
    \end{cases}
\end{equation}
where the coefficients $\Theta_{t,j}$ are determined recursively by
\begin{equation}
    \Theta_{t,j} = \left( \Gamma_{\bm{W}}(t+1,t-j+1) - \sum_{k=0}^{t-j-1} \Theta_{t,t-k} \Sigma_{k+1} \Theta_{t-j,t-j-k}^\top \right) \Sigma_{t-j+1}^{-1}, \quad \text{for } j\in[t], \, t\in[2,T]
\end{equation}
\end{lemma}
\begin{proof}[Proof of Lemma~\ref{lemma:inno_alg_W}]
    The result follows directly from Proposition 11.4.2 in \citet{brockwell1991time} by applying it to the forward-transformed process $\{\bm{W}_t\}$.
\end{proof}

\section{Supplementary Material: Sufficient Statistics for $\mathrm{VARMA}(1,1)$}
\label{append.proofs_varma11}

As we mentioned, sufficient statistics for univariate $\mathrm{ARMA}$ processes were derived by \citet{esa1998sufficient}, while \citet{kharrati2009sufficient} derived explicit forms for the multivariate $\mathrm{VARMA}(1,1)$ setting under fixed parameters. We extend their results by directly applying Theorem~\ref{thm.varma_pq} with $\mathrm{VARMA}(1,1)$, which admits a closed-form recursive representation for the sufficient statistics of $\bm{\xi} = (\Phi, \Theta, \Sigma_\epsilon)$. 

Under Assumption~\ref{assump: causal_invert} and~\ref{assump:sym_cov}, there exists a white-noise sequence $\{\bm{\tilde{\epsilon}}_t\}\sim \mathrm{WN}(\mathbf{0},\,\Sigma_\epsilon)$ such that 
\[
\bm{Y}_t-\Phi \bm{Y}_{t+1}=\bm{\tilde{\epsilon}}_t+\Theta \bm{\tilde{\epsilon}}_{t+1}.
\]

This representation enables the construction of a forward-transformed process $\{\bm{W}_t\}$ defined by
\[
\bm{W}_t =
\begin{cases}
    \bm{Y}_t-\Phi \bm{Y}_{t+1}, & t\in[T-1],\\
    \bm{Y}_T, &  t=T,
\end{cases} 
\]

The process $\{\bm W_t\}$ admits a recursive innovation representation, which plays a central role in deriving the likelihood and sufficient statistics. Let $\Gamma_{\bm{W}}(t,t-h)$ denote the lag-$h$ covariance of $\bm{W}_t$, then
\[
    \Gamma_{\bm{W}}(t,t-h)=
        \begin{cases}
            \Theta \Sigma_{\epsilon} \Theta^\top + \Sigma_{\epsilon}, & h=0,\, t\in[T-1],\\
            \Gamma_{\bm{Y}}(0),  & h=0,\,t=T,\\
            \Sigma_{\epsilon}\Theta^\top, & h=1,\,t\in[T],\\
            \bm{0}_{n\times n},& \text{otherwise},
        \end{cases}
\]
where $\Gamma_{\bm{Y}}(0)$ denotes the lag-zero covariance of $\bm{Y}_t$. Define the innovation as $\bm{R}_t = \bm{W}_t - \hat{\bm{W}}_t$ with covariance $\Sigma_t$. Then the one-step predictor of $\{\bm{W}_t\}$ is given by
\[
    \hat{\bm{W}}_t=
    \begin{cases}
        \bm{0}, & t=1,\\
        \Theta_{t-1,1}\bm{R}_{t-1},& t\in[2,T],
    \end{cases}
\]
which can be computed recursively by $\{\Theta_{t,1}\}$ and $\{\Sigma_t\}$, where
$\Theta_{t,1}=\Sigma_\epsilon \Theta^\top \Sigma_{t}^{-1}$, for $t\in[T-1]$ and
\[
    \Sigma_t = 
    \begin{cases}
         \Theta \Sigma_{\epsilon} \Theta^\top + \Sigma_{\epsilon}, & t=1, \\
         \Theta \Sigma_{\epsilon} \Theta^\top + \Sigma_{\epsilon} - \Theta_{t-1,1} \Sigma_{t-1} \Theta_{t-1,1}^\top,  & t\in[2,T-1],\\
         \Gamma_{\bm{Y}}(0)-\Theta_{T-1,1} \Sigma_{T-1} \Theta_{T-1,1}^\top, & t=T.
    \end{cases}
\]

The innovations $\{\bm{R}_{t}\}$ can be expressed in terms of $\{\bm{\alpha}_t\}$ and $\{\bm{\beta}_t\}$. Specifically, we have
\[
    \bm{R}_{t}=   
        \begin{cases}
            \bm{\alpha}_1-\Phi \bm{\beta}_1 &\quad t=1, \\
             \bm{\alpha}_t-\Phi\bm{\beta}_t+\sum_{j=0}^{t-2} C_{j,t}\Phi\bm{\beta}_{t-1-j} &\quad t\in[2, T-1],\\
             \bm{\alpha}_{T}+\Theta_{T-1,1}\Phi\bm{\beta}_{T-1}-\Theta_{T-1,1}\sum_{j=0}^{T-3} C_{j,T-1}\Phi\bm{\beta}_{T-2-j} &\quad t=T,
        \end{cases}   
\]
where
\[
\begin{aligned}
    & \bm{\alpha}_1=\bm{Y}_1, \quad
    \bm{\alpha}_t=\bm{Y}_t-\Theta_{t-1,1}\bm{\alpha}_{t-1}, &&\quad \text{for } t\in[2,T], \\
    & \bm{\beta}_t=\bm{Y}_{t+1}, &&\quad \text{for } t\in[T-1], \\
    & C_{j,t}=(-1)^{2t-j} \Theta_{t-1,1} \times \cdots \times \Theta_{t-j-1,1}, &&\quad \text{for } j\in [0,T-3], t\in[2,T-1].
\end{aligned}
\]

Under Assumptions~\ref{assump: causal_invert} and \ref{assump:sym_cov}, the likelihood function has the following decomposition:
\[
L(\bm{Y};\bm{\xi})=g_0(\bm{Y}) g_1(\hat{\bm{\xi}}(\bm{Y}),\bm{\xi}),
\]
where 
\[
\hat{\bm{\xi}}(\bm{Y}) =  \left( (\bm Y_t \bm Y_{t-j}^{\top})_{j=0}^{t-1}\right)_{t=1}^{T}
\]
forms a set of sufficient statistics for $\bm{\xi}$ and
\[
    g_0(\bm{Y})= (2\pi)^{-\frac{nT}{2}}, \quad
    g_1(\hat{\bm{\xi}},\bm{\xi})
    = \left(\prod_{t=1}^{T} \det(\Sigma_t)\right)^{-\frac{1}{2}} \exp\left\{-\frac{h(\hat{\bm{\xi}},\bm{\xi})}{2}  \right\}.
\]
with $h(\hat{\bm{\xi}},\bm{\xi}) = \sum_{t=1}^T \bm{R}_{t}^\top \Sigma_t^{-1} \bm{R}_{t}$.


\section{Supplementary Material: Estimation of \texorpdfstring{$\Gamma_{\bm{Y}_t}(0)$}{Gamma(Y_t)(0)}}
\label{append.gamma_y}

Consider a $\mathrm{VARMA}(p,q)$ model with parameters $\bm{\xi}=(\{\Phi_i\}_{i=1}^p, \{\Theta_i\}_{i=1}^q, \Sigma_\epsilon)$, given by
\[
    \bm{Y}_t = \Phi_1 \bm{Y}_{t-1} + \cdots + \Phi_p \bm{Y}_{t-p}  + \bm{\epsilon}_t + \Theta_1 \bm{\epsilon}_{t-1} + \Theta_q \bm{\epsilon}_{t-q}.
\]
Define the augmented state vector $\bm{y}_t$ as
\[
    \bm{y}_t = P\bm{y}_{t-1}+U_t,  
\]
where
\[
    \bm{y}_t
    =\begin{pmatrix}
    \bm{Y}_t,\\
    \bm{Y}_{t-1},\\
    \vdots\\
    \bm{Y}_{t-p+1},\\
    \bm{\epsilon}_t, \\
    \bm{\epsilon}_{t-1},\\
    \vdots\\
    \bm{\epsilon}_{t-q+1} 
    \end{pmatrix}
    ,\quad
    P= 
    \begin{pmatrix}
    \Phi_1 & \cdots & \Phi_{p-1}  & \Phi_{p} & \Theta_1 & \cdots & \Theta_{q-1} & \Theta_q \\[1mm]
    I_n & \cdots & \bm{0} & \bm{0} & \bm{0} & \cdots & \bm{0} & \bm{0} \\[1mm]
    \vdots & \ddots & \vdots & \bm{0} & \vdots & \ddots & \vdots & \vdots \\[1mm]
    \bm{0} & \cdots & I_n & \bm{0} & \bm{0} & \cdots & \bm{0} & \bm{0} \\[1mm]
    \bm{0} & \cdots & \bm{0} & \bm{0} & \bm{0} & \cdots & \bm{0} & \bm{0} \\[1mm]
    \bm{0} & \cdots & \bm{0} & \bm{0} & I_n & \cdots & \bm{0} & \bm{0}  \\[1mm]
    \vdots & \ddots & \vdots & \vdots  & \vdots & \ddots  & \vdots & \vdots \\[1mm]
    \bm{0} & \cdots & \bm{0} & \bm{0} & \bm{0} & \cdots & I_n & \bm{0}  \\[1mm]
    \end{pmatrix}
    ,\quad
    U_t
    =\begin{pmatrix}
    \epsilon_t\\
    \bm{0}\\
    \vdots\\
    \bm{0}\\
    \epsilon_t \\
    \bm{0}\\
    \vdots\\
    \bm{0}
    \end{pmatrix}
    .
\]
Let $\Sigma_U= B \Sigma_\epsilon B^\top$, where $B = (I_n, \bm{0},...,\bm{0}, I_n, \bm{0},...,\bm{0})^\top$. 
Then the stationary covariance matrix of the augmented state, denoted by $S=\mathbb{E}[\bm{y}_t\bm{y}_t^\top]$, that satisfies
\[
\mathrm{vec} (S) = \left( I_{(p+q)^2n^2} - P \otimes P \right)^{-1} \mathrm{vec}(\Sigma_U),
\]
Then the covariance matrix $\Gamma_{\bm{Y}_t}(0)$ is given by the top-left $n\times n$ block matrix of $S$.

\textbf{Estimation of $\Gamma_{\bm Y_t}(0)$ for $\mathrm{VARMA}(1,1)$.}
For the special case of a $\mathrm{VARMA}(1,1)$ model, the covariance matrix $\Gamma_{\bm{Y}_t}(0)$can be obtained by considering an augmented process
that includes a second lag of $\bm{Y}_t$ with $\Phi_2=\bm{0}$; see, e.g.,  \citet{lutkepohl2005new}. Specifically, consider the process
\[
    \bm{Y}_t=\Phi_1 \bm{Y}_{t-1}+\Phi_2 \bm{Y}_{t-2}+\Theta_1 \bm{\epsilon}_{t-1}+\bm{\epsilon}_t.
\]
Define 
\[
    \bm{y}_t
        =\begin{pmatrix}
            \bm{Y}_{t}\\
            \bm{Y}_{t-1}\\
            \bm{\epsilon}_t 
        \end{pmatrix}
    ,\quad
    P
    =\begin{pmatrix}
            \Phi_1 & \bm{0} & \Theta_1 \\[1mm]
            I_n & \bm{0} &\bm{0} \\[1mm]
            \bm{0} & \bm{0} & \bm{0}
    \end{pmatrix}
    , \quad
    U_t
    =\begin{pmatrix} 
        \bm{\epsilon}_t\\ 
        \bm{0}\\
        \bm{\epsilon}_t 
    \end{pmatrix}
    ,\quad
    \Sigma_U
    =\begin{pmatrix}
        \Sigma_\epsilon & \bm{0} & \Sigma_\epsilon \\[1mm]
        \bm{0} & \bm{0} &\bm{0} \\[1mm]
        \Sigma_\epsilon & \bm{0} & \Sigma_\epsilon
    \end{pmatrix}
    .
\]
Then, $\Gamma_{\bm{Y}_t}(0)$ is characterized by
\[
    \mathrm{vec} 
            \begin{pmatrix}
                \Gamma_{\bm{Y}_t}(0) & \Gamma_{\bm{Y}_t}(1) & \Sigma_\epsilon \\[1mm]
                \Gamma_{\bm{Y}_t}(-1) & \Gamma_{\bm{Y}_t}(0) &\bm{0} \\[1mm]
                \Sigma_\epsilon & \bm{0} & \Sigma_\epsilon
            \end{pmatrix}
        = \left( I_{9n^2} - P \otimes P \right)^{-1} \mathrm{vec}(\Sigma_U).
\]

\section{Supplementary Material: Implementation of PTO, ETO, and FPtP}\label{append.method_alg}

\textbf{PTO.} The PTO approach follows the classical two-stage paradigm in which a predictive model is first trained to forecast the next-period realization of the uncertain parameter, and the resulting point prediction is subsequently used as a plug-in estimate in the optimization problem. We implement four predictive learners: a Recurrent Neural Network (RNN), a Long Short-Term Memory network (LSTM), a Random Forest (RF), and XGBoost (XGB). Our implementations rely on standard machine learning libraries with fixed hyperparameters, as summarized below.
\begin{itemize}
    \item PTO-RNN\,/\,PTO-LSTM: Implemented in PyTorch, with hidden dimension set to 32. Models are trained using the Adam optimizer with $10^{-3}$ learning rate and mean squared error loss.
    \item PTO-RF: Implemented using scikit-learn’s \texttt{RandomForestRegressor} with 100 trees, maximum depth 10, and all cores enabled.
    \item PTO-XGB: Implemented using 
    \texttt{XGBRegressor} with the configuration parameters (tree\_method="hist", objective="reg:squarederror",
    eval\_metric="rmse",
    n\_estimators=100,\,
    max\_depth=6,\,
    learning\_rate=0.1,\,
    subsample=0.8,\,
    colsample\_bytree=0.8,\,
    min\_child\_weight=1,\,
    seed=42).
\end{itemize}
Given a trained predictor $f_\theta$, the PTO decision is obtained by computing a point forecast $\hat{\bm Y}_{T+1} = f_\theta(\bm Y_{T-S+1:T})$, forming the corresponding matrix $D_2(\hat{\bm Y}_{T+1})$, and solving the resulting deterministic optimization problem. Algorithm~\ref{alg:pto} summarizes the procedure.

\begin{algorithm}
\caption{\textsc{MethodSolve}: PTO}
\label{alg:pto}
\begin{algorithmic}[1]
\REQUIRE
sample length $T$;
sample $\bm{Y}_{1:T}$;
sequence length $S$;
learner $f_\theta$ (e.g., RNN/LSTM);
learning rate $\eta$;
constants $\mu_0, D_1,\bm{e},\bm{x}^0$

\vspace{2mm}
\STATE Initialize model parameters $\theta$.

\vspace{1mm}
\FOR{$t = S$ {\bfseries to} $T$}
    \STATE Construct input-output pair
    $(\bm{Y}_{t-S:t-1},\, \bm{Y}_t)$.
    \STATE Update $\theta$ using Adam optimizer on loss $\big\| f_\theta(\bm{Y}_{t-S:t-1}) - \bm{Y}_t \big\|_2^2$
\ENDFOR

\vspace{1mm}
\STATE $\hat{\bm Y}_{T+1} \gets f_\theta(\bm{Y}_{T-S+1:T})$
\STATE $D_2^{\scriptscriptstyle{\text{PTO}}} \gets D_2(\hat{\bm{Y}}_{T+1})$
\RightComment{via \eqref{eq.D2}}
\STATE $\bm x_{\scriptscriptstyle{\text{PTO}}} \gets \left( D_1 + D_2^{\scriptscriptstyle{\text{PTO}}} \right)^{-1} \left( \mu_0 D_1 \bm e + D_2^{\scriptscriptstyle{\text{PTO}}} \bm x^0 \right)$

\vspace{2mm}
\STATE {\bfseries Return } $\bm{x}_{\scriptscriptstyle{\text{PTO}}}$
\end{algorithmic}
\end{algorithm}

\textbf{ETO.}
The ETO approach first estimates the underlying data-generating process via maximum likelihood estimation and then optimizes with respect to the estimated distribution. Algorithm~\ref{alg:eto} summarizes the procedure.

\begin{algorithm}
\caption{\textsc{MethodSolve}: ETO}
\label{alg:eto}
\begin{algorithmic}[1]
\REQUIRE
sample length $T$;
sample $\bm{Y}_{1:T}$;
constants $\mu_0, D_1,\bm{e},\bm{x}^0$

\vspace{2mm}
\STATE $\hat{\bm{\xi}}_{\scriptscriptstyle{\text{ETO}}} \gets \arg\max_{\bm{\xi}} L(\bm{Y}_{1:T};\bm{\xi})$
\RightComment{via \eqref{eq.lkh_11}}

\vspace{1mm}
\STATE $D_2^{\scriptscriptstyle{\text{ETO}}} \gets \mathbb{E}_{\bm{Y}_{T+1}|\hat{\bm{\xi}}_{\scriptscriptstyle{\text{ETO}}}} [D_2(\bm{Y}_{T+1})]$
\RightComment{via \eqref{eq.expected_D2}}
\STATE $x_{\scriptscriptstyle{\text{ETO}}} = \left( D_1 + D_2^{\scriptscriptstyle{\text{ETO}}} \right)^{-1} \left( \mu_0 D_1 \bm e + D_2^{\scriptscriptstyle{\text{ETO}}} \bm x^0 \right)$

\vspace{2mm}
\STATE {\bfseries Return } $\bm{x}_{\scriptscriptstyle{\text{ETO}}}$
\end{algorithmic}
\end{algorithm}

\textbf{FPtP.} We study the FPtP approach using two tree-based learners: Random Forest (hereafter denoted as `FPtP-RF') and XGBoost (hereafter denoted as `FPtP-XGB'). We do not implement FPtP for neural network architectures (e.g., RNNs or LSTMs), and therefore do not consider the corresponding ETO models. The reason is methodological rather than computational: \citet{bertsimas2020predictive} do not provide a principled strategy for generating multiple predictive outcomes beyond the PTO framework. While several heuristic approaches have been discussed in the literature (e.g., injecting noise into covariates), there is currently no consensus on a theoretically grounded strategy. To avoid drawing potentially misleading conclusions, we deliberately exclude these models from our analysis. Algorithm~\ref{alg:fptp} summarizes the FPtP procedure considered in this paper.

\underline{FPtP-RF.} For Random Forests, the $M$ predictive samples correspond to the outputs of the individual trees, denoted by $\hat{\bm Y}^{(m)}_{T+1}$, given a new observation $\bm z$ or the historical window $\bm{Y}_{T-S+1:T}$ in the pesudo code. In this setting, each base learner $f^{(m)}$ corresponds to a single decision tree. \citet{bertsimas2020predictive} propose two weighting schemes: (i) Uniform weights, where $w_m(\bm z)=1/M$ for all $m$ and we denote this variant `FPtP-RF-U'; (ii) Voting weights, based on each tree’s contribution to the standard ensemble prediction, which we denote as `FPtP-RF-V' (`V' for voting). 

\underline{FPtP-XGB.} We adopt XGBoost as a representative boosting method. Following \citet{bertsimas2020predictive}, multiple predictive samples are generated by considering prefix models obtained after the first $m$ boosting rounds. Accordingly, $\hat{\bm Y}^{(m)}_{T+1}$ denotes the prediction produced at iteration $m$, and we employ uniform weights $w_m(\bm z)=1/M$. In this case, $f^{(m)}$ represents the prefix model containing the first $m$ boosted trees. 

\begin{algorithm}
\caption{\textsc{MethodSolve}: FPtP}
\label{alg:fptp}
\begin{algorithmic}[1]
\REQUIRE
sample length $T$;
sample $\bm{Y}_{1:T}$;
sequence length $S$;
tree-based learner $f_\theta$ (e.g., RF/XGB);
number of ensembles $M$;
constants $\mu_0, D_1,\bm{e},\bm{x}^0$

\vspace{2mm}
\STATE $\mathcal{D} \gets \{(\bm{Y}_{t-S:t-1},\bm{Y}_t)\}_{t=S}^{T}$
\STATE $\theta \gets \argmin_{\theta} \sum_{(\bm u,\bm v)\in\mathcal{D}}\big\|f_{\theta}(\bm u)-\bm v\big\|_2^2$
\STATE $f_{\theta} \gets \{f^{(m)}\}_{m=1}^{M}$

\vspace{1mm}
\FOR{$m=1$ {\bfseries to} $M$}
    \STATE $\hat{\bm Y}^{(m)}_{T+1} \gets f^{(m)}(\bm{Y}_{T-S+1:T})$ 
\ENDFOR

\vspace{1mm}
\STATE $D_2^{\scriptscriptstyle{\text{FPtP}}} \gets \frac{1}{M}\sum_{m=1}^M D_2(\hat{\bm{Y}}_{T+1}^{(m)})$
\RightComment{via \eqref{eq.D2}}
\STATE $\bm x_{\scriptscriptstyle{\text{FPtP}}} \gets \left( D_1 + D_2^{\scriptscriptstyle{\text{FPtP}}} \right)^{-1} \left( \mu_0 D_1 \bm e + D_2^{\scriptscriptstyle{\text{FPtP}}} \bm x^0 \right)$

\vspace{2mm}
\STATE {\bfseries Return } $\bm{x}_{\scriptscriptstyle{\text{FPtP}}}$
\end{algorithmic}
\end{algorithm}

\section{Supplementary Material: Details of Experimental Setting}\label{append.exp}
\subsection{Computational strategies}

\textbf{Assumption on $\bm \xi$.} 
For a $\mathrm{VARMA}(p,q)$ model with parameters $\bm{\xi} = (\{\Phi_i\}_{i=1}^p, \{\Theta_i\}_{i=1}^q, \Sigma_\epsilon)$,
we assume that $\bm{\xi}$ satisfies Assumptions~\ref{assump: causal_invert} and~\ref{assump:sym_cov}.
The stationarity and invertibility conditions in Assumption~\ref{assump: causal_invert} are ensured by requiring that all eigenvalues of the companion matrices $M_\Phi$ and $M_\Theta$ lie strictly
inside the unit circle, i.e., $|\lambda_i(M_\Phi)| < 1$ and $|\lambda_i(M_\Theta)| < 1$, where
\[
M_\Phi=
\begin{pmatrix}
\Phi_1 & \Phi_2 & \Phi_3 & \cdots & \Phi_p\\
I_k & 0 & 0 & \cdots & 0\\
0 & I_k & 0 & \cdots & 0\\
\vdots & & \ddots & & \vdots\\
0 & 0 & \cdots & I_k & 0
\end{pmatrix} 
\quad\text{and}\quad
M_\Theta=
\begin{pmatrix}
-\Theta_1 & -\Theta_2 & -\Theta_3 & \cdots & -\Theta_q\\
I_k & 0 & 0 & \cdots & 0\\
0 & I_k & 0 & \cdots & 0\\
\vdots & & \ddots & & \vdots\\
0 & 0 & \cdots & I_k & 0
\end{pmatrix}.
\]

To guarantee Assumption~\ref{assump:sym_cov}, we further impose that all components of $\bm{\xi}$ are symmetric and mutually commutative.
\begin{lemma}
\label{lemma:sym_cov}
If the parameters $\bm{\xi}$ are symmetric and commutative, then
Assumption~\ref{assump:sym_cov} holds; that is,
\[
\Phi^{-1}(z)\Theta(z)\Sigma_\epsilon
\Theta^{\top}(z^{-1})\Phi^{\top-1}(z^{-1})
=
\Phi^{-1}(z^{-1})\Theta(z^{-1})\Sigma_\epsilon
\Theta^{\top}(z)\Phi^{\top-1}(z).
\]
\end{lemma}
\begin{proof}[Proof of Lemma~\ref{lemma:sym_cov}]
Since $\bm{\xi}$ are symmetric and commutative, they can be simultaneously
diagonalized by an orthogonal matrix $P$, such that
\[
\Lambda_{\Phi_j} = P^\top \Phi_j P, \quad \Lambda_{\Theta_i} = P^\top \Theta_i P, \quad \Lambda_{\Sigma_\epsilon} = P^\top \Sigma_\epsilon P.
\]
Define
\[
P^\top \Phi(z) P 
= I-\sum_{j=1}^p \Lambda_{\Phi_j} z^j
:=D_\Phi(z) , \quad
P^\top \Theta(z) P 
= I+\sum_{i=1}^q \Lambda_{\Theta_i} z^i
:=D_\Theta(z).
\]
Then
\[
\Phi(z) = P D_\Phi(z) P^\top, \quad
\Phi^{-1}(z) = (P D_\Phi(z) P^\top)^{-1} = P D_\Phi^{-1}(z) P^\top
\]
and
\[
\Theta(z) = P D_\Theta(z) P^\top, \quad
\Theta(z^{-1}) =  P D_\Theta(z^{-1}) P^\top.
\]
Using these representations, the left-hand side becomes
\[
\begin{aligned}
\text{LHS}
&= \Phi^{-1}(z)\Theta(z)\Sigma\Theta^{\top}(z^{-1})\Phi^{\top-1}(z^{-1}) \\
&= P D_\Phi^{-1}(z) P^\top P D_\Theta(z) P^\top  P \Lambda_{\Sigma_\epsilon} P^\top P  D_\Theta(z^{-1}) P^\top P D_\Phi^{-1}(z^{-1}) P^\top \\
&= P \left( D_\Phi^{-1}(z) D_\Theta(z) \Lambda_{\Sigma_\epsilon} D_\Theta(z^{-1}) D_\Phi^{-1}(z^{-1}) \right)P^\top,
\end{aligned}
\]
while the right-hand side is
\[
\begin{aligned}
\text{RHS}
&= \Phi^{-1}(z^{-1})\Theta(z^{-1})\Sigma\Theta^{\top}(z)\Phi^{\top-1}(z) \\
&= P D_\Phi^{-1}(z^{-1}) P^\top P D_\Theta(z^{-1}) P^\top P \Lambda_{\Sigma_\epsilon} P^\top P D_\Theta(z) P^\top P  P^\top \\
&= P \left( D_\Phi^{-1}(z^{-1})  D_\Theta(z^{-1}) \Lambda_{\Sigma_\epsilon} D_\Theta(z) D_\Phi^{-1}(z) \right)P^\top.
\end{aligned}
\]
Since all matrices inside the parentheses are diagonal and therefore commute,
we conclude that $\text{LHS} = \text{RHS}$.
\end{proof}

\textbf{Likelihood under symmetric and commutative $\bm \xi$.} 
Here, We specialize to a $\mathrm{VARMA}(1,1)$ process with parameter $\bm{\xi} = (\Phi, \Theta, \Sigma_\epsilon)$. Suppose that $\bm{\xi}$ takes the symmetric and commutative form
\begin{equation}\label{eq.sym_xi}
    \Theta = \theta I_n, \quad
    \Phi = P \Lambda_{\Phi} P^{\top}, \quad
    \Sigma_{\epsilon}=P \Lambda_{\Sigma_{\epsilon}} P^{\top}
\end{equation}
where $P$ is an orthogonal matrix. 

Under this restriction, the likelihood in Appendix~\ref{append.proofs_varma11} admits a
simplified expression.

\begin{proposition}\label{prop.sym_xi}
Under restriction~\eqref{eq.sym_xi}, the likelihood
$L(\bm{Y}; \bm{\xi})$ in Appendix~\ref{append.proofs_varma11} can be written as
\[
    L(\bm{Y};\bm{\xi})=g_0(\bm{Y}) g_1(\hat{\bm{\xi}}(\bm{Y}),\bm{\xi}),
\]
where 
\[
    g_0(\bm{Y})= (2\pi)^{-\frac{nT}{2}}, \quad 
    g_1(\hat{\bm{\xi}}(\bm{Y}),\bm{\xi})=\left(\prod_{t=1}^{T} \det(\tilde{\Sigma}_t) \right)^{-\frac{1}{2}} \exp\left\{-\frac{1}{2} h(\hat{\bm{\xi}}(\bm{Y}),P, \theta, \Lambda_\Phi, \Lambda_{\Sigma_\epsilon}) \right\}.
\]
More specifically,
\[
\begin{aligned}
    &h(\hat{\bm{\xi}},P, \theta, \Lambda_\Phi, \Lambda_{\Sigma_\epsilon})
    =(\tilde{\bm{\alpha}}_1 - \Lambda_{\Phi} \tilde{\bm{\beta}}_1)^\top \tilde{\Sigma}_1^{-1} (\tilde{\bm{\alpha}}_1 - \Lambda_{\Phi} \tilde{\bm{\beta}}_1) \\ 
    &\quad + \sum_{t=2}^{T-1} \left( \tilde{\bm{\alpha}}_t - \Lambda_{\Phi} \tilde{\bm{\beta}}_t + \sum_{j=0}^{t-2} C_{j,t} \Lambda_{\Phi} \tilde{\bm{\beta}}_{t-1-j} \right)^\top \tilde{\Sigma}_t^{-1} \left( \tilde{\bm{\alpha}}_t - \Lambda_{\Phi} \tilde{\bm{\beta}}_t + \sum_{j=0}^{t-2} C_{j,t} \Lambda_{\Phi} \tilde{\bm{\beta}}_{t-1-j} \right) \\
    &\quad + \left( \tilde{\bm{\alpha}}_T + \Theta_{T-1,1} \Lambda_{\Phi} \tilde{\bm{\beta}}_{T-1} - \Theta_{T-1,1}\sum_{j=0}^{T-3} C_{j,T-1} \Lambda_{\Phi} \tilde{\bm{\beta}}_{T-2-j} \right)^\top \tilde{\Sigma}_T^{-1}  \\
    &\qquad \qquad \qquad \qquad \qquad \qquad \qquad
    \left( \tilde{\bm{\alpha}}_T + \Theta_{T-1,1} \Lambda_{\Phi} \tilde{\bm{\beta}}_{T-1} - \Theta_{T-1,1}\sum_{j=0}^{T-3} C_{j,T-1} \Lambda_{\Phi} \tilde{\bm{\beta}}_{T-2-j} \right).
\end{aligned}
\]
The covariance determinants and inverses are given by
\[
    \det(\tilde{\Sigma}_t)= 
    \begin{cases}
    \det(s_t \Lambda_{\Sigma_{\epsilon}}), &\quad t\in[T-1] \\
    \det(\Lambda_{\Gamma_{\bm{Y}}(0)}-\frac{\theta^2}{s_{T-1}} \Lambda_{\Sigma_{\epsilon}}), &\quad t=T
    \end{cases}
    \,\, \text{and} \,\,
    \tilde{\Sigma}_t^{-1} 
    = 
    \begin{cases}
    \frac{1}{s_t} \Lambda_{\Sigma_{\epsilon}}^{-1}, &\quad t\in[T-1] \\
    (\Lambda_{\Gamma_{\bm{Y}}(0)}-\frac{\theta^2}{s_{T-1}} \Lambda_{\Sigma_{\epsilon}})^{-1}, &\quad t=T.
    \end{cases}
\]
Here,
\[
    \begin{aligned}
        & s_1 = 1+\theta^2, \quad 
            s_t = s_1-\theta d_t, \quad
            d_t = \frac{\theta}{s_{t-1}}, \quad
            \Theta_{t-1,1} = d_t I_n, &&\quad \text{for } t\in[2,T], \\
        & \bm{\alpha}_1=\bm{Y}_1, \quad
        \bm{\alpha}_t=\bm{Y}_t-d_t\bm{\alpha}_{t-1}, \quad 
        \tilde{\bm{\alpha}}_t = P^\top \bm{\alpha}_t, &&\quad \text{for } t\in[2,T], \\
        & \bm{\beta}_t=\bm{Y}_{t+1}, \quad
        \tilde{\bm{\beta}}_t = P^\top \bm{\beta}_t, &&\quad \text{for } t\in[T-1], \\
        & C_{j,t}=(-1)^{2t-j} d_t\times \cdots \times d_{t-j} I_n, &&\quad \text{for } j\in [0,T-3], t\in[2,T-1].
    \end{aligned}
\]
\end{proposition}

\begin{proof}[Proof of Proposition~\ref{prop.sym_xi}]
Substituting~\eqref{eq.sym_xi} into the innovation covariance matrices $\Sigma_t$ in Appendix~\ref{append.proofs_varma11} yields $\Theta_{t-1,1} = d_t I_n$ and 
\[
\Sigma_t = 
\begin{cases}
 s_t \Sigma_{\epsilon}, &\quad t\in[T-1] \\
 \Gamma_{\bm{Y}}(0)-\frac{\theta^2}{s_{T-1}} \Sigma_{\epsilon}, &\quad t=T
\end{cases}
\]
Since orthogonal transformations preserve determinants and quadratic forms, we apply $P^\top(\cdot)P$ to $\Sigma_t$, $\bm{\alpha}_t$, and $\bm{\beta}_t$, and obtain
\[
\tilde{\Sigma}_t = P^\top \Sigma_t P =
\begin{cases}
s_t \Lambda_{\Sigma_\epsilon}, & t \in [T-1], \\[1mm]
\Lambda_{\Gamma_{\bm{Y}}(0)}
- \dfrac{\theta^2}{s_{T-1}} \Lambda_{\Sigma_\epsilon}, & t = T.
\end{cases}
\]

The stated expressions for $\tilde{\Sigma}_t^{-1}$,
$\det(\tilde{\Sigma}_t)$, and the quadratic form
$h(\cdot)$ then follow directly.
\end{proof}

\subsection{Evaluation on Synthetic Data}\label{append.exp.syn}

\textbf{Generating oracle parameter $\bm \xi$ and sample $\bm Y$.} 

\underline{Well-specified setting.}
Each oracle parameter
$\bm{\xi} = (\Phi, \Theta, \Sigma_\epsilon)$
is generated according to the following procedure:
\begin{itemize}
    \item Construct an invertible MA matrix
    $\Theta = \theta I_n$ with eigenvalue $|\theta| \le 0.9$.
    \item Sample an orthogonal matrix $P$ uniformly at random.
    \item Construct a stationary AR matrix
    $\Phi = P \Lambda_{\Phi} P^\top$ with eigenvalues
    $|\lambda_{\Phi}| < 0.9$.
    \item Construct a positive definite covariance matrix
    $\Sigma_{\epsilon} = P \Lambda_{\Sigma_{\epsilon}} P^\top$
    with eigenvalues $\lambda_{\Sigma_\epsilon} \in (0.1, 0.9)$.
\end{itemize}

To generate a diverse collection of oracle parameters, we initialize a reference parameter $\bm{\xi}^0$ and construct a symmetric prior distribution centered at $\bm{\xi}^0$. Candidate parameters are weighted according to their squared Frobenius distance from $\bm{\xi}^0$. Specifically, we define
\[
u_k
=
\frac{\exp(-d_k)}{\sum_j \exp(-d_j)},
\qquad
d_k
=
\|\Theta_k - \Theta^0\|_F^2
+
\|\Phi_k - \Phi^0\|_F^2
+
\|\Sigma_{\epsilon,k} - \Sigma_\epsilon^0\|_F^2 .
\]

To construct this prior, we first sample $N_t = 10,000$ candidate parameters $(\bm{\xi}_1, \ldots, \bm{\xi}_{N_t})$ in a neighborhood of $\bm{\xi}^0$, retaining only those that satisfy Assumptions~\ref{assump: causal_invert} and~\ref{assump:sym_cov}. The normalized weights
$\bm{u} = (u_1, \ldots, u_{N_t})$ then define a discrete prior over the candidate set.

For evaluation, we draw $N_o = 50$ oracle parameters independently from this prior. In addition, the OVE method requires Monte Carlo sampling from the same prior. To avoid overlap between testing parameters and Monte Carlo samples, we draw a separate set of $N_{\text{ove}}$ parameters
$(\bm{\xi}_1^{\scriptscriptstyle{\text{OVE}}}, \ldots,
 \bm{\xi}_{N_{\text{ove}}}^{\scriptscriptstyle{\text{OVE}}})$
exclusively for OVE estimation.

Given an oracle parameter $\bm{\xi}$, we generate $N_s$ independent sample paths
$\bm{Y}_{1:T}$ of length $T = 25$.

\input{tables/config_param}

\underline{Mis-specified setting.}
Oracle parameters $\bm{\xi} = (\{\Phi_i\}_{i=1}^p, \{\Theta_i\}_{i=1}^q, \Sigma_\epsilon)$ are generated from a $\mathrm{VARMA}(p,q)$ process using a procedure analogous to that of the $\mathrm{VARMA}(1,1)$ case:
\begin{itemize}
    \item Construct invertible MA matrices
    $\Theta_i = \theta_i I_n$ with $|\theta_i| \le 0.9$.
    \item Sample an orthogonal matrix $P$ uniformly at random.
    \item Construct stationary AR matrices
    $\Phi_i = P \Lambda_{\Phi_i} P^\top$ with
    $|\lambda_{\Phi_i}| < 0.9$.
    \item Construct a positive definite covariance matrix
    $\Sigma_{\epsilon} = P \Lambda_{\Sigma_{\epsilon}} P^\top$
    with eigenvalues $\lambda_{\Sigma_\epsilon} \in (0.1, 0.9)$.
\end{itemize}

The prior weights are defined as

\[
u_k
=
\frac{\exp(-d_k)}{\sum_j \exp(-d_j)},
\qquad
d_k
=
\sum_{i=1}^q\|\Theta_{k,i} - \Theta^0\|_F^2
+
\sum_{i=1}^p\|\Phi_{k,i} - \Phi^0\|_F^2
+
\|\Sigma_{\epsilon,k} - \Sigma_\epsilon^0\|_F^2 .
\]

Given an oracle parameter $\bm{\xi}$, we again generate $N_s$ independent sample paths $\bm{Y}_{1:T}$ of length $T = 25$. 

Since we focus on small model mis-specification, the learning model continues to assume a $\mathrm{VARMA}(1,1)$ structure. To construct the prior distribution over OVE parameters, for each $\mathrm{VARMA}(p,q)$ oracle with weight $u_k$, we first generate $N_{\text{sove}}$ sample paths of length $T = 25$. We then estimate the corresponding $\mathrm{VARMA}(1,1)$ parameters via maximum likelihood estimation and assign each estimate a prior weight of $u_k / N_{\text{sove}}$.

\input{tables/mis}

\textbf{Portfolio optimization setup.}
For all synthetic experiments, The total fund size is set to $A = 1$, and the initial portfolio position $\bm{x}^0 \in (0,1)^n$ is drawn independently from a uniform distribution. The expected excess return vector $\bm{e}$ is generated independently with each component sampled uniformly from $[0.05,\,0.15]$. The risk-aversion parameter is fixed at $\gamma = 0.1$, and the return variance is set to $\delta^2 = 0.1$.
In addition, we set $\mu_2 = 0.1$ throughout all synthetic experiments.
Unless otherwise specified, these parameters are held fixed across both
well-specified and mis-specified settings.

\textbf{Discussion on results.} 
Tables~\ref{tab:well} and~\ref{tab:mis} summarize the performance of all methods under the well-specified and mis-specified settings, respectively, in terms of computational time, predictive accuracy (MSE), and decision quality (relative regret). 

In the well-specified setting (Table~\ref{tab:well}), for the case $n=10$, the computational cost of ETO becomes prohibitive when directly maximizing the likelihood defined in OVE to obtain the MLE. Therefore, for ETO under $n=10$, we instead estimate the model parameters using the \texttt{VARMAX} implementation in the \texttt{statsmodels} package.

\subsection{Evaluation on Real Data}\label{append.exp.real}

\textbf{Real-data experimental setup.}
We evaluate the proposed methods using real-world financial data from the
\emph{World Stock Prices (Daily Updating)} dataset, obtained from
\url{https://www.kaggle.com/datasets/nelgiriyewithana/world-stock-prices-daily-updating}.
Our analysis focuses on four stocks—GOOGL, HMC, AXP, and BAMXF—over the sample period from January~1,~2010, to January~1,~2019, excluding the COVID-19 period to avoid the influence of extreme market volatility. 

For each stock, we define the dollar trading volume as the product of the
reported trading volume and the OHLC4 price, where OHLC4 is computed as the
average of the open, high, low, and close prices.
The data are aggregated at a biweekly frequency using a rolling window of
length $\tau=10$, and a logarithmic transformation is applied to obtain the
log biweekly dollar trading volume series used in the empirical analysis.

The portfolio consists of $n=4$ assets with an initial allocation
$\bm{x}^0 \in (0,1)^n$ drawn uniformly at random. We set the total fund size to $A=1$ and consider an observation horizon of length $T=10$. The mean excess return vector $\bm{e}$ and return variance are estimated from historical data, with the average variance denoted by $\delta^2$. The risk-aversion parameter is fixed at $\gamma=0.1$ throughout the experiments.

\textbf{Stationarity and $\mathrm{VARMA}$ model validation.}
To verify the stationarity assumption required by the $\mathrm{VARMA}$ framework, we conduct both Augmented Dickey--Fuller (ADF) and KPSS tests on each transformed series. As reported, all four stocks reject the unit root null hypothesis under the ADF test at conventional significance levels, while failing to reject the null of stationarity under the KPSS test. These results jointly indicate that the log biweekly dollar trading volume series are stationary.

\input{tables/aic}

We further assess the appropriateness of the $\mathrm{VARMA}(1,1)$ specification by comparing it with alternative $\mathrm{VARMA}(p,q)$ models using standard information criteria, including AIC, BIC, and HQIC. Among the candidate models considered, $\mathrm{VARMA}(1,1)$ achieves the lowest AIC and provides a favorable trade-off between model fit and complexity. This empirical evidence supports the modeling assumption adopted in the main text that the multivariate time series of log biweekly dollar trading volume aggregates is well described by a stationary $\mathrm{VARMA}(1,1)$ process.

\textbf{Rolling Evaluation on Real Data.} 
We evaluate all methods on real-world financial time series using a rolling out-of-sample procedure designed to closely mimic practical decision-making settings.
The overall evaluation pipeline is summarized in Algorithm~\ref{alg:real} and consists of four main components: time-series aggregation, model training, rolling testing, and regret computation.

First, we preprocess the raw daily dollar trading volume data by aggregating it into $\tau$ shifted log biweekly time series, as described in Algorithm~\ref{alg:real_agg}. Second, for each aggregated series, we train all baseline models—including PTO, ETO, and FPtP—using a fixed training horizon, as detailed in Algorithm~\ref{alg:real_train}.
For OVE, we construct its parameter support by applying ETO to rolling training windows and then generate a prior distribution over the resulting parameter candidates via cross-validated kernel density estimation (Algorithm~\ref{alg:real_prior}). Third, model performance is evaluated using a rolling testing window.
At each testing step, all methods generate decisions based on historical observations of length $T$, and their realized costs are evaluated using a sample-average approximation (SAA) of the true objective. An oracle decision is computed by directly minimizing the same SAA objective, serving as a benchmark. Finally, performance is summarized by the average relative regret across all rolling testing windows, defined as the relative gap between each method’s realized cost and the oracle cost.

\input{algorithm/real}
\input{algorithm/real_agg}
\input{algorithm/real_train}
\input{algorithm/real_prior}

%% file: tables/config_param.tex
\begin{table}[ht]
\caption{Configuration parameter setting.}
\label{tab:config}
\vskip -0.15in
\renewcommand{\arraystretch}{1.2}
\begin{center}
\begin{small}
\begin{sc}
\begin{tabular}{lcccccccc}
\toprule
\multicolumn{1}{c}{Setting}     & $n$ & $N_t$  & $N_o$ & $N_{\text{ove}}$ & $N_{\text{sove}}$ & $N_s$ & $N_{ts}$ & $T$ \\ \hline
\multirow{3}{*}{Well-specified} & 2   & 10,000 & 50    & 500              & -                 & 200   & -      & 25  \\
                                & 5   & 10,000 & 50    & 1,000            & -                 & 200   & -      & 25  \\
                                & 10  & 10,000 & 50    & 1,000            & -                 & 200   & -      & 25  \\ \hline
Mis-specified                   & 5   & 10,000 & 50    & 1,000            & 5                 & 200   & -      & 25  \\ \hline
Real-world                      & 4   & 910    & 1     & 200              & -                 & -     & 114    & 10  \\ 
\bottomrule
\end{tabular}
\end{sc}
\end{small}
\end{center}
\end{table}

%% file: tables/mis.tex
\begin{table}[htbp]
\caption{Summary results of mis-specified setting.}
\label{tab:mis}
\centering
\renewcommand{\arraystretch}{1.2}
\begin{small}
\begin{sc}
\resizebox{\linewidth}{!}{%
\begin{tabular}{lcccccccccccc}
\toprule
\multicolumn{1}{c}{\multirow{2}{*}{Model}} & \multicolumn{4}{c}{Time (second)} & \multicolumn{4}{c}{MSE}               & \multicolumn{4}{c}{Relative Regret (\%)} \\ \cmidrule(lr){2-5} \cmidrule(lr){6-9} \cmidrule(lr){10-13}
\multicolumn{1}{c}{}                       & $(2,1)$  & $(3,1)$  & $(1,2)$  & $(1,3)$  & $(2,1)$   & $(3,1)$   & $(1,2)$   & $(1,3)$   & $(2,1)$   & $(3,1)$   & $(1,2)$  & $(1,3)$  \\ \hline
A-OVE         & 5.66 & 5.49 & 5.42 & 5.21 & -      & -      & -      & -      & \textbf{0.50}  & \textbf{0.80}  & 0.54 & \textbf{0.37} \\
ETO         & 4.47 & 4.53 & 4.39 & 4.27 & 28.79 & 31.46 & 26.00 & 25.97 & 0.58  & 0.98  & \textbf{0.47} & \textbf{0.37} \\ \hline
PTO-RNN     & \textbf{0.25} & \textbf{0.27} & \textbf{0.27} & \textbf{0.27} & \textbf{20.16} & \textbf{20.18} & \textbf{23.45} & \textbf{23.09} & 24.59  & 22.94  & 37.41 & 27.55 \\
PTO-LSTM    & 0.28 & 0.31 & 0.33 & 0.32 & 22.41 & 22.67 & 26.47 & 25.24 & 13.80  & 11.80  & 20.13 & 16.01 \\
PTO-RF      & 0.45 & 0.42 & 0.44 & 0.44 & 25.29 & 23.84 & 30.94 & 30.11 & 27.56  & 22.28  & 42.57 & 28.40 \\
PTO-XGB     & 3.69 & 3.19 & 3.06 & 2.86 & 28.13 & 26.54 & 35.46 & 34.14 & 44.06  & 40.93  & 66.01 & 47.60 \\ \hline
FPtP-RF     & 0.45 & 0.42 & 0.44 & 0.44 & 25.29 & 23.84 & 30.94 & 30.11 & 9.25 & 5.75  & 12.33 & 7.90 \\
FPtP-XGB    & 3.69 & 3.19 & 3.06 & 2.86 & 28.13 & 26.54 & 35.46 & 34.14 & 24.58  & 23.50  & 37.36 & 27.23 \\ 
\bottomrule
\end{tabular}}
\end{sc}
\end{small}
\end{table}

%% file: tables/aic.tex
\begin{table}[htbp]
\centering
\caption{$\mathrm{VARMA}$ model comparison based on information criteria}
\label{tab:varma_ic}
\renewcommand{\arraystretch}{1.2}
\begin{small}
\begin{sc}
\begin{tabular}{lccccc}
\toprule
\multicolumn{1}{c}{Model} & $p$ & $q$ & AIC & BIC & HQIC \\ \hline
$\mathrm{VARMA}(0,1)$ & 0 & 1 & 879.96 & 982.44 & 921.32  \\
$\mathrm{VARMA}(1,0)$ & 1 & 0 & 822.96 & \textbf{925.45} & 864.33  \\
$\mathrm{VARMA}(1,1)$ & 1 & 1 & \textbf{784.23} & 941.37 & \textbf{847.65}  \\
$\mathrm{VARMA}(2,0)$ & 2 & 0 & 817.26 & 974.41 & 880.69  \\
$\mathrm{VARMA}(2,1)$ & 2 & 1 & 798.92 & 1010.72 & 884.40  \\
$\mathrm{VARMA}(3,0)$ & 3 & 0 & 801.14 & 1012.94 & 886.62  \\
$\mathrm{VARMA}(3,1)$ & 3 & 1 & 805.14 & 1071.59 & 912.68 \\ \bottomrule
\end{tabular}
\end{sc}
\end{small}
\end{table}

%% file: algorithm/real.tex
\begin{algorithm}
\caption{Evaluation on Real Data}
\label{alg:real}
\begin{algorithmic}[1]
\REQUIRE
daily dollar volume $\bm{Y}^d$;
aggregation window size $\tau$;
training horizon $T_{\text{train}}$;
testing horizon $T_{\text{test}}$;
sample length $T$

\vspace{2mm}

\STATE $\mathcal{Y}_{\text{agg}} \gets \textsc{AggTimeSeries}(\bm{Y}^d)$
\RightComment{prepare the aggregated time series}

\STATE $\mathcal{M} \gets \textsc{TrainModel}(\mathcal{Y}_{\text{agg}})$
\RightComment{training the models}

\FOR{$\bm Y^{(i)}$ {\bfseries in} $\mathcal{Y}_{\text{agg}}$}
    \STATE $\bm Y^{(i)}_{\text{test}} \gets \bm Y^{(i)}_{T_{\text{train}}+1:T_{\text{train}}+T_{\text{test}}+1}$
\ENDFOR

\STATE $N_{ts} \gets T_{\text{test}} - (T+1) +1$
\RightComment{number of testing samples}

\FOR{$l=1$ {\bfseries to} $N_{ts}$}
    \FOR{$i=1$ {\bfseries to} $\tau$}
        \STATE $\bm{Y}_{1:T}^{(l,i)} \gets (\bm Y^{(i)}_{\text{test}})_{l:l+T-1}$
        \STATE $\bm{Y}_{T+1}^{(l,i)} \gets (\bm Y^{(i)}_{\text{test}})_{l+T}$
    \ENDFOR
    
    \STATE Define $\rho^{l}(\bm x) := \frac{1}{\tau}\sum_{i=1}^{\tau} \pi(\bm x, \bm{Y}_{T+1}^{(l,i)})$ 
    \RightComment{SAA estimation of $l$-th true cost}
    
    \FOR{$\cdot$ {\bfseries in} $\mathcal{M}$}
        \FOR{$i=1$ {\bfseries to} $\tau$}
            \STATE Execute model $\mathcal{M}_{\cdot}^{(i)}$ on $\bm{Y}_{1:T}^{(l,i)}$ and obtain decisions $\bm{x}_l^{(i)}$
            \STATE $\rho_\cdot^{l,i} \gets \rho^{l}(\bm{x}_l^{(i)})$ 
            \RightComment{SAA estimation of model performance}
        \ENDFOR
        \STATE $\sigma_\cdot^l \gets \frac{1}{\tau}\sum_{i=1}^{\tau} \rho_\cdot^{l,i}$
        \RightComment{average model performance}
    \ENDFOR

    \STATE $\bm x_*^l \gets \argmin_{\bm x} \rho^{l}(\bm x)$
    \RightComment{oracle decision}
    \STATE $\rho_*^l \gets \rho^{l}(\bm x_*^l)$
    \RightComment{oracle cost}
\ENDFOR

\vspace{2mm}
\STATE {\bfseries Return } $\left\{\frac{1}{N_{ts}}\sum_{l=1}^{N_{ts}}\frac{\sigma_\cdot^l-\rho_*^l}{\rho_*^l}\right\}_{\cdot\in\mathcal{M}}$
\end{algorithmic}
\end{algorithm}

%% file: algorithm/real_agg.tex
\begin{algorithm}
\caption{\textsc{AggTimeSeries}}
\label{alg:real_agg}
\begin{algorithmic}[1]
\REQUIRE
daily dollar volume $\bm{Y}^d$;
aggregation window size $\tau$;
target aggregated horizon $T_{\text{agg}}$;
mean of log-aggregated volumes $\bm{\mu}$

\vspace{2mm}
\FOR{$i = 1$ {\bfseries to} $\tau$}
    \FOR{$t = 1$ {\bfseries to} $T_{\text{agg}}$}
        \STATE $\bm Y^{(i)}_t \gets \log\left(\sum\limits_{s = \tau (t-1) + i}^{\tau t + i - 1} \bm Y_{s}\right) - \bm{\mu}$
        \RightComment{$\tau$-day volumes aggregation starting at shift $i$}
    \ENDFOR
\ENDFOR

\vspace{2mm}
\STATE {\bfseries Return } $\mathcal{Y}_{\text{agg}} = \{\bm Y^{(i)}\}_{i=1}^\tau$
\end{algorithmic}
\end{algorithm}

%% file: algorithm/real_train.tex
\begin{algorithm}
\caption{\textsc{TrainModel}}
\label{alg:real_train}
\begin{algorithmic}[1]
\REQUIRE
aggregated time series $\mathcal{Y}_{\text{agg}}$;
aggregation window size $\tau$;
training horizon $T_{\text{train}}$;
sample length $T$;
number of OVE parameters $N_{\text{ove}}$

\vspace{2mm}
\FOR{$\bm Y^{(i)}$ {\bfseries in} $\mathcal{Y}_{\text{agg}}$}
    \STATE $\bm Y^{(i)}_{\text{train}} \gets \bm Y^{(i)}_{1:T_{\text{train}}}$
    \STATE Obtain model $\mathcal{M}^{(i)}_{\cdot}$ by training on $\bm Y^{(i)}_{\text{train}}$
    \RightComment{applies to PTO, ETO, and FPtP}

    \STATE $\Xi^{(i)} \gets \emptyset$
    \RightComment{generate OVE parameter support}
    \FOR{$k=1$ {\bfseries to} $T_{\text{train}}-T+1$}
        \STATE $\bm Y^{(i,k)} \gets \bm Y^{(i)}_{k:k+T-1}$
        \STATE $\bm{\xi}^{(i,k)}$ is obtained by applying ETO to $\bm Y^{(i,k)}$
        \STATE $\Xi^{(i)} \gets \Xi^{(i)}\cup\bm{\xi}^{(i,k)}$
    \ENDFOR
\ENDFOR

\STATE $\Xi_{\scriptscriptstyle{\text{OVE}}} \gets \bigcup\limits_{i=1}^{\tau} \Xi^{(i)}$
\STATE $u \gets \textsc{GenPrior}(\Xi_{\scriptscriptstyle{\text{OVE}}})$
\STATE Select $N_{\text{ove}}$ candidates from $\Xi_{\scriptscriptstyle{\text{OVE}}}$ according to prior $u$ and obtain $\mathcal{M}_{\scriptscriptstyle{\text{OVE}}}$

\vspace{2mm}
\STATE {\bfseries Return } $\mathcal{M} = \left\{\{\mathcal{M}^{(i)}_{\scriptscriptstyle{\text{PTO}}}\}_{i=1}^\tau, \{\mathcal{M}^{(i)}_{\scriptscriptstyle{\text{ETO}}}\}_{i=1}^\tau, \{\mathcal{M}^{(i)}_{\scriptscriptstyle{\text{FPtP}}}\}_{i=1}^\tau, \mathcal{M}_{\scriptscriptstyle{\text{OVE}}}\right\}$
\end{algorithmic}
\end{algorithm}

%% file: algorithm/real_prior.tex
\begin{algorithm}
\caption{\textsc{GenPrior}}
\label{alg:real_prior}
\begin{algorithmic}[1]
\REQUIRE
Support $\Xi_{\scriptscriptstyle{\text{OVE}}}$;
Candidate bandwidths $\mathcal{B}$;
Number of folds $K$ for cross-validation

\vspace{2mm}
\STATE $N_t \gets |\Xi_{\scriptscriptstyle{\text{OVE}}}|$
\STATE $\Lambda_{\bm\xi} \gets \emptyset$
\RightComment{eigenvalue set of $\bm{\xi}$}

\vspace{1mm}
\FOR{$\bm{\xi}$ {\bfseries in} $\Xi_{\scriptscriptstyle{\text{OVE}}}$}
    \STATE $\bm{\lambda}_{\bm{\xi}} \gets (\bm{\lambda}_{\Theta},\bm{\lambda}_{\Phi}, \bm{\lambda}_{\Sigma_\epsilon})$
    \STATE $\Lambda_{\bm\xi} \gets \bm{\lambda}_{\bm{\xi}}$
\ENDFOR

\vspace{1mm}
\STATE Initialize $J^{*} \gets -\infty$, $b^{*} \gets 0$
\FOR{$b$ {\bfseries in} $\mathcal{B}$}
    \STATE Partition $\{1,\ldots,N_t\}$ into $K$ disjoint folds $\{\mathcal{D}_k\}_{k=1}^K$
    \FOR{$k = 1$ {\bfseries to} $K$}
        \STATE $\mathcal{D}_{-k} \gets \bigcup_{r \neq k} \mathcal{D}_r$
        \STATE 
        $\hat{u}_b^{(-k)}(\bm{\lambda})
        \gets
        \frac{1}{|\mathcal{D}_{-k}|\,b^{|\bm{\lambda}|}}
        \sum_{j \in \mathcal{D}_{-k}}
        K\!\left(
            \frac{\bm{\lambda} - \bm{\lambda}_j}{b}
        \right)$
        \STATE 
        $L_k(b)
        \gets
        \frac{1}{|\mathcal{D}_k|}
        \sum_{i \in \mathcal{D}_k}
        \log
        \hat{u}_b^{(-k)}(\bm{\lambda}_i)$
    \ENDFOR
    \STATE $\text{CV}(b)\gets \frac{1}{K} \sum_{k=1}^{K} L_k(b)$
    \IF{$\text{CV}(b) > J^\star$}
        \STATE $J^\star \gets \text{CV}(b)$
        \STATE $b^\star \gets b$
    \ENDIF
\ENDFOR

\vspace{1mm}

\STATE $d_{\bm{\xi}} \gets \hat{u}_{b^\star}(\lambda_{\bm{\xi}}),\quad \forall\, \bm{\xi} \in \mathcal{S}_o$
\STATE $u_{\bm{\xi}} \gets \frac{d_{\bm{\xi}}}{\sum_{\bm{\xi}' \in \mathcal{S}_o} d_{\bm{\xi}'}}, \quad \forall\, \bm{\xi} \in \mathcal{S}_o$

\vspace{2mm}
\STATE \textbf{Return } $\{u_{\bm{\xi}}\}_{\bm{\xi} \in \mathcal{S}_{o}}$
\end{algorithmic}
\end{algorithm}